%% file: main.tex
\documentclass{article}
\usepackage{calibhyper_arxiv,times}

\usepackage{amsmath,amssymb,amsthm,bm}

\input{math_commands.tex}

\usepackage{graphicx}
\usepackage{booktabs,array}
\usepackage{float}
\usepackage{placeins}
\usepackage{xcolor}
\usepackage{hyperref}
\hypersetup{
  colorlinks=true,linkcolor=black,citecolor=black,urlcolor=black,
  pdftitle={CalibHyper: Chance-Corrected Relational Hypergraphs for Few-Shot Molecular Property Prediction},
  pdfauthor={Linyu Li; Zhi Jin; Yuanpeng He; Dongming Jin; Huanyu Liu; Huanyao Zhang; Haoran Duan; Heng Tian; Gadeng Luosang; Nyima Tashi},
  pdfsubject={Few-shot molecular property prediction},
  pdfkeywords={Molecular property prediction, Few-shot learning, Hypergraphs, Chance correction}
}
\usepackage{url}
\usepackage{etoc}
\definecolor{tocblue}{rgb}{0.0,0.2,0.6}

\newtheorem{proposition}{Proposition}
\newtheorem{lemma}{Lemma}[section]
\theoremstyle{definition}
\newtheorem{definition}{Definition}[section]

\title{CalibHyper: Chance-Corrected\\
Relational Hypergraphs for Few-Shot\\
Molecular Property Prediction}

\author{%
\begin{minipage}{\textwidth}
\centering
Linyu Li\textsuperscript{1,2}\quad
Zhi Jin\textsuperscript{1,2,3}\quad
Yuanpeng He\textsuperscript{1}\quad
Dongming Jin\textsuperscript{1,2}\\[2pt]
Huanyu Liu\textsuperscript{1}\quad
Huanyao Zhang\textsuperscript{1}\quad
Haoran Duan\textsuperscript{3}\quad
Heng Tian\textsuperscript{3}\\[2pt]
Gadeng Luosang\textsuperscript{4}\quad
Nyima Tashi\textsuperscript{4}\endgraf
\vspace{7pt}
{\small
\textsuperscript{1}School of Computer Science, Peking University\\
\textsuperscript{2}Key Laboratory of High Confidence Software Technologies,\\
Peking University, Ministry of Education\\
\textsuperscript{3}School of Computer Science, Wuhan University\\
\textsuperscript{4}School of Information Science and Technology,\\
Tibetan Language Intelligence National Key Laboratory, Tibet University\endgraf
}
\end{minipage}%
}
\date{}
\begin{document}

\maketitle

\begin{abstract}
   Molecular property prediction is central to drug development and materials discovery, but experiments are costly and labeled data are scarce. Context-aware methods use auxiliary assay labels to support few-shot prediction, and recent work supervises property relations with label agreement. However, label agreement is sensitive to class marginals and does not directly capture dependence between properties. We propose CalibHyper, a chance-corrected relational hypergraph method based on the joint label distribution. CalibHyper subtracts an independence baseline from the ordered four-state label distribution and shrinks the residual according to the number of joint observations. A swap-equivariant relation head estimates these residuals, which choose the auxiliary properties for each molecule and set the sign and weight of their hyperedge messages. On thirteen datasets from five benchmarks, in both 1-shot and 10-shot settings, CalibHyper and its ablation settings achieve ROC-AUC competitive with the strongest reported results.
\end{abstract}

\etocdepthtag.toc{mtmain}
\section{Introduction}
\label{sec:1}

Molecular property prediction is a fundamental task in drug development and materials discovery \citep{wu2018,david2020}, but experimental labels are expensive to obtain \citep{wang2025b}. Few-shot methods reuse knowledge learned on existing properties and adapt to an unseen property from a small support set \citep{altaetran2017,guo2021,chen2023,wu2024,li2025}. Labels from existing auxiliary assays can also serve as context for prediction (Figure~\ref{fig:1}A). GS-Meta \citep{zhuang2023} and KRGTS \citep{wang2025a} organize this information with graphs that connect molecules to properties and with relations among tasks, while methods for auxiliary data selection and task weighting consider how well the auxiliary information matches the target task \citep{huang2023,zhong2025}.

ReCoG \citep{wang2026} supervises relation learning with binary label agreement, which mixes the dependence between two properties with their class marginals. When positives are rare, two independent properties still agree on almost every molecule, because most molecules are negative for both (Figure~\ref{fig:1}B): MUV-713 and MUV-733 agree on 99.8718\% of their 3,121 jointly labeled molecules, against 99.8719\% expected under independence (Cohen's $\kappa=-0.0005$). Relation supervision should therefore separate chance agreement, which the marginals explain, from dependence beyond it.

Binary agreement also merges the two ordered disagreement states, target positive with auxiliary negative and the reverse (Figure~\ref{fig:1}C). The association itself can change with molecular structure (Figure~\ref{fig:1}D), and an estimate from few joint observations is less reliable than one from many (Figure~\ref{fig:1}E). In ReCoG, the relation head affects prediction only through an auxiliary loss on shared representations, and a gate weights each auxiliary property once per episode. We instead let relation estimates drive messages that differ from molecule to molecule.

We propose CalibHyper, a chance-corrected relational hypergraph method whose two modules share one relation estimate. C4RC (chance-corrected four-state relation calibration) keeps the four ordered label states, subtracts the independence baseline computed on jointly observed molecules, shrinks the residual according to the number of joint observations, and fits it with a swap-equivariant head. MSRHA (molecule-conditioned signed relational hypergraph adapter) uses the predicted residuals to select auxiliary properties for each molecule and to set the sign and weight of messages along ordered hyperedges of a molecule, the target, and an auxiliary property.

\begin{figure}[t]
\centering
\includegraphics[width=\linewidth]{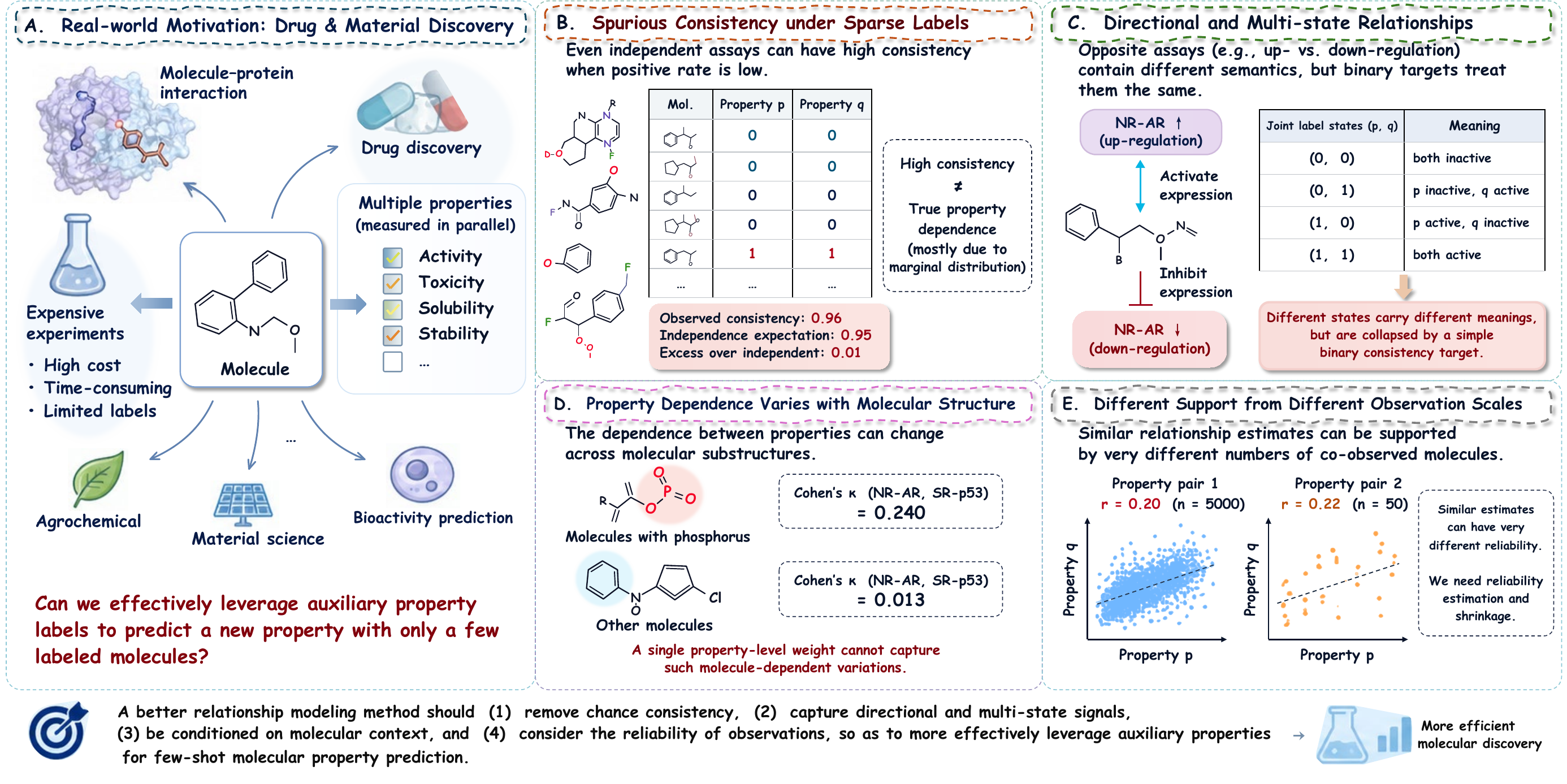}
\caption{Motivation. (A) Few-shot prediction with auxiliary assays. (B) Chance agreement under rare positives. (C) Binary targets merge states $01$ and $10$. (D) Structure-dependent association in Tox21. (E) Similar estimates from unequal numbers of joint observations. B, C, and E are schematic.}
\label{fig:1}
\end{figure}

Our contributions are as follows.
\begin{itemize}
\item Using real molecular label matrices, we show that binary agreement supervision confounds class marginals with property dependence and that swap-invariant relation inputs cannot represent ordered label states.
\item We propose CalibHyper, which couples four-state relation calibration with signed hypergraph propagation so that relation estimates drive molecule-conditioned auxiliary messages. We characterize the calibrated target and prove that the relation head is swap-equivariant and that the adapter is the identity at initialization.
\item On 26 conditions (thirteen datasets, 1-shot and 10-shot), CalibHyper and its ablation settings achieve ROC-AUC competitive with the strongest reported results, with larger margins in the 1-shot setting. Ablations and interventions on auxiliary labels and routing show how each component and the available context affect performance.
\end{itemize}

\section{Related Work}
\label{sec:2}

\paragraph{Few-shot and context-aware molecular property prediction.}
Few-shot methods transfer supervision from existing properties to new ones through shared representations, sample matching, and task adaptation \citep{wang2025b}, as in PAR \citep{wang2021}, ADKF-IFT \citep{chen2023}, PACIA \citep{wu2024}, Pin-Tuning \citep{wang2024}, and UniMatch \citep{li2025}. Multitask models share statistical strength across the property labels of the same molecule \citep{ramsundar2015,mayr2018}, and context-aware methods extend this idea to unseen properties by exploiting context molecules, measurements, or context graphs, as in MHNfs \citep{schimunek2023}, CAMP \citep{fifty2023}, GS-Meta \citep{zhuang2023}, KRGTS \citep{wang2025a}, and CaMol \citep{hoang2026}; graph prompting methods such as OFA \citep{liu2024} organize tasks with prompt graphs. The closest work, ReCoG \citep{wang2026}, supervises property relations with binary agreement and compresses auxiliary property nodes conditioned on the target. CalibHyper replaces binary agreement with chance-corrected four-state residuals and uses them to generate molecule-conditioned messages.

\paragraph{Selecting and compressing auxiliary information.}
MolGroup \citep{huang2023} and AutAuT \citep{zhong2025} select or weight auxiliary information at the dataset and task level. Information bottleneck methods such as CGIB \citep{lee2023}, IGIB-ISE \citep{zhang2025}, and InfoAlign \citep{liu2025a} constrain what a representation retains. We combine auxiliary-node compression with a chance-corrected relation target.

\paragraph{Chance-corrected agreement and higher-order relations.}
Cohen's $\kappa$ \citep{cohen1960} separates observed agreement from the chance agreement implied by the marginals, and under skewed marginals high agreement can coexist with low $\kappa$ \citep{feinstein1990,byrt1993}; we carry this correction into relation supervision as a four-state residual. Hypergraph learning \citep{feng2019,chien2022,feng2025,zhang2026,gao2026}, signed networks \citep{derr2018}, and dynamic message passing \citep{zhang2020} handle higher-order connections, signed relations, and input-dependent neighborhoods. Our hyperedges are ordered triples of a molecule, the target, and an auxiliary property, and one supervised relation score sets their selection, weight, and sign. Appendix~\ref{sec:A} gives an extended review.

\section{Problem Setup and Notation}
\label{sec:3}

Let $\mathcal{M}$ be a set of molecules and $\mathcal{T}$ a set of binary property prediction tasks, split into disjoint meta-training and meta-test sets, $\mathcal{T}_{\text{train}}\cap\mathcal{T}_{\text{test}}=\emptyset$. The label matrix is $Y\in\{0,1,\star\}^{|\mathcal{M}|\times|\mathcal{T}|}$, where $y_{i,p}=1$ means that molecule $m_i$ is active for property $p$, $0$ means inactive, and $\star$ means unmeasured. We write $o_{i,p}=\1[y_{i,p}\neq\star]$ for the observation indicator.

An episode $E_\tau=(S_\tau,Q_\tau)$ targets property $\tau$. The support set $S_\tau$ usually contains $K$ molecules per class, and the query set $Q_\tau$ contains $M$ molecules. The auxiliary set is $A_\tau\subset\mathcal{T}_{\text{train}}\setminus\{\tau\}$ with $N=|A_\tau|$. The context graph $G_\tau$ has molecule nodes $V_\tau^{m}$ from the episode and property nodes $V_\tau^{p}=\{\tau\}\cup A_\tau$; observed labels determine the edge types, and query labels of the target are excluded. Molecule features come from a pretrained GIN \citep{xu2019,hu2020}, $x_i=F_{\text{mol}}(m_i)\in\mathbb{R}^{d}$, and property nodes use learnable embeddings. A two-layer context encoder produces $H=\{h_v\}_{v\in V_\tau}=\mathrm{ContextEncoder}(G_\tau,X)$, and the predictor outputs $\hat y_{i,\tau}=F_{\text{pred}}([\,h'_i\Vert h'_\tau\,])$, where $H'$ are the representations after adaptation ($H'=H$ without an adapter). A context information bottleneck compresses auxiliary property representations conditioned on the target \citep{lee2023,wang2026} and assigns each auxiliary property a gate $\lambda_a\in(0,1)$ through a Gumbel-sigmoid relaxation \citep{jang2017,maddison2017}. Training uses MAML-style bilevel optimization \citep{finn2017}: the inner loop adapts on the support loss, and the outer loop updates shared parameters with the query loss and auxiliary objectives.

The joint label of a property pair $(p,q)$ is one-hot encoded as $e(u,v)\in\{0,1\}^4$ with coordinates $(e_{00},e_{01},e_{10},e_{11})$, where the first digit refers to $p$. The appendix gives the notation, formal definitions, Propositions~\ref{prop:degenerate} to~\ref{prop:identity}, the proofs of all propositions, and further analysis.

\section{Limits of Binary Relation Supervision}
\label{sec:4}

ReCoG \citep{wang2026} regresses the disagreement indicator $y^{(i,p,q)}_{\text{rel}}=(y_{i,p}-y_{i,q})^2\in\{0,1\}$ of each property pair with an MSE loss.

\paragraph{Chance agreement.}
Let $\pi_p$ and $\pi_q$ be the positive rates of $p$ and $q$ on their jointly observed molecules, $a_{pq}$ their observed agreement, and $c_{pq}=(1-\pi_p)(1-\pi_q)+\pi_p\pi_q$ the agreement expected under independence.

\begin{proposition}[Chance decomposition]
\label{prop:chance}
If $c_{pq}<1$, then $\mathbb{E}[y_{\text{rel}}]=1-a_{pq}=(1-c_{pq})(1-\kappa_{pq})$, where $\kappa_{pq}=(a_{pq}-c_{pq})/(1-c_{pq})$ is Cohen's $\kappa$. With $s_{i,p}=1-2y_{i,p}$, the four-state indicators satisfy
\begin{equation}
\label{eq:contrast}
(e_{00}+e_{11})-(e_{01}+e_{10})=s_{i,p}\,s_{i,q},
\end{equation}
whose mean exceeds its independence value $(1-2\pi_p)(1-2\pi_q)$ by $4\,\mathrm{Cov}(y_p,y_q)=2(a_{pq}-c_{pq})$.
\end{proposition}

When positives are rare, $c_{pq}\approx1$, so dependence moves the binary target only in proportion to $1-c_{pq}$. For MUV-713 and MUV-733, $1-c_{pq}=0.0013$, so each unit of $\kappa_{pq}$ changes the expected binary target by only 0.0013. Figure~\ref{fig:2} shows that this regime is common. The mean $\kappa$ is 0.010 on MUV and 0.032 on PCBA, where joint negatives make up on average more than 99.99\% and 99.8\% of agreeing label pairs. Negative dependence is also frequent: 20.4\%, 8.7\%, and 15.6\% of meta-training pairs have $\kappa<-0.02$ in the ToxCast subsets NVS, BSK, and CEETOX. Evidence is uneven as well. The median number of joint observations is 235{,}660 on PCBA but 22 on NVS, and two pairs with $\hat\kappa\approx0.22$ differ by about 1{,}200-fold in their sensitivity to a single relabeled molecule (Figure~\ref{fig:2}f).

\begin{figure}[t]
\centering
\includegraphics[width=\linewidth]{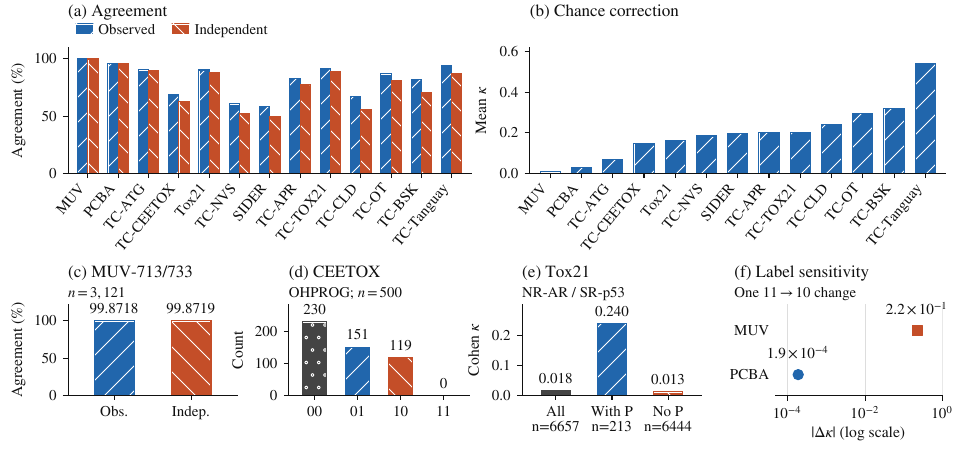}
\caption{Label statistics (Appendix~\ref{sec:H}). (a) Observed and independence agreement and (b) mean Cohen's $\kappa$ per dataset (TC: ToxCast). (c) MUV-713 and MUV-733. (d) A CEETOX increase and decrease pair. (e) $\kappa$ with and without phosphorus. (f) $|\Delta\kappa|$ after one $11\to10$ relabeling.}
\label{fig:2}
\end{figure}

\paragraph{Swap-invariant relation input.}
ReCoG's relation head reads $u=[h_i\Vert h_p]\odot[h_i\Vert h_q]=[h_i\odot h_i\Vert h_p\odot h_q]$. The first half ignores the property pair, the second half ignores the molecule, and $u$ is unchanged when $p$ and $q$ are swapped. A swap-equivariant four-state output computed from $u$ must therefore assign equal values to the states $01$ and $10$, and its squared error on an ordered target is at least $\frac12(r_{01}-r_{10})^2$ (Proposition~\ref{prop:degenerate}). The two states do differ in practice: for the CEETOX assays in Figure~\ref{fig:2}d they contain 151 and 119 molecules, which the binary label merges.

\section{Method}
\label{sec:5}

CalibHyper has two modules (Figure~\ref{fig:3}). C4RC learns reliability-weighted four-state residuals with a swap-equivariant head $F_{\text{rel}}$ (Section~\ref{sec:5.1}), and MSRHA turns them into molecule-conditioned signed hyperedge messages that update the representations used for prediction (Section~\ref{sec:5.2}). The backbone follows Section~\ref{sec:3}.

\begin{figure}[t]
\centering
\includegraphics[width=\linewidth]{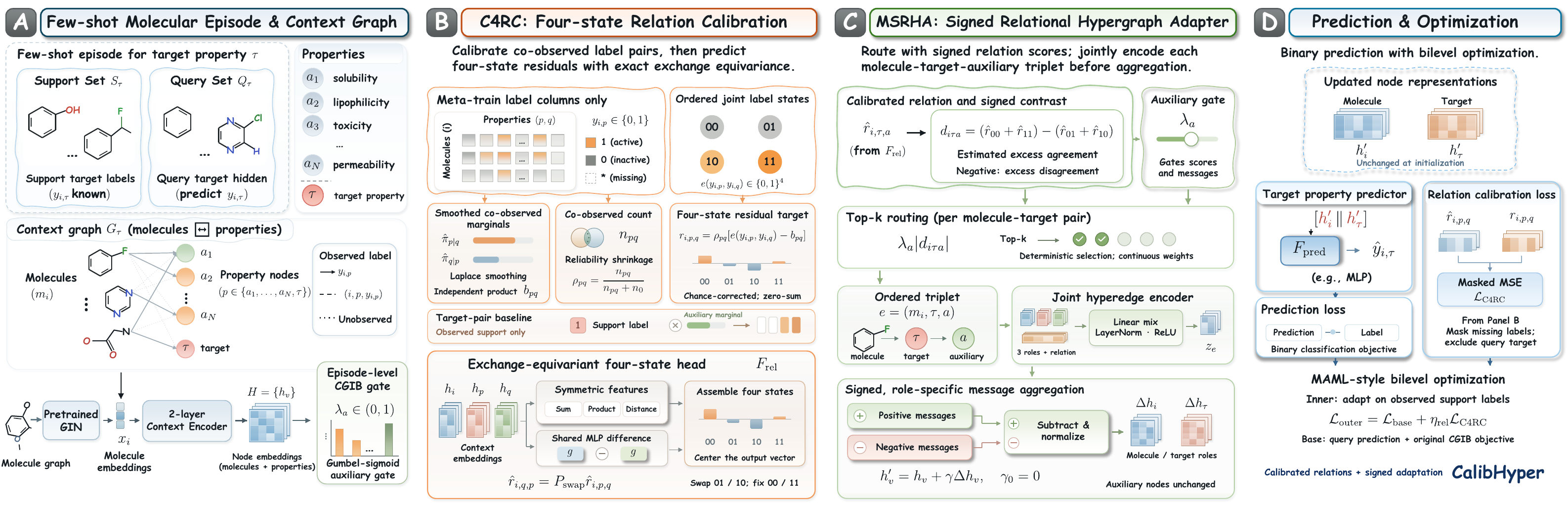}
\caption{Architecture. (A) Context graph and encoder. (B) C4RC: chance-corrected four-state relation head. (C) MSRHA: signed messages on top-$k$ hyperedges. (D) Prediction and training.}
\label{fig:3}
\end{figure}

\subsection{C4RC: Chance-Corrected Four-State Relation Calibration}
\label{sec:5.1}

\paragraph{Independence baseline.}
For an ordered pair $(p,q)$ of meta-training properties, let $I_{pq}=\{i: o_{i,p}=o_{i,q}=1\}$, $n_{pq}=|I_{pq}|$, and $n^+_{p|q}=\sum_{i\in I_{pq}}y_{i,p}$. With a smoothing constant $\alpha$ (1 by default), the positive rate of $p$ on the jointly observed molecules is $\hat\pi_{p|q}=(n^{+}_{p|q}+\alpha)/(n_{pq}+2\alpha)$, and $\hat\pi_{q|p}$ is defined symmetrically. Two independent properties with these marginals have the four-state distribution
\begin{equation*}
b_{pq}=\Big[(1-\hat\pi_{p|q})(1-\hat\pi_{q|p}),\;
(1-\hat\pi_{p|q})\hat\pi_{q|p},\;
\hat\pi_{p|q}(1-\hat\pi_{q|p}),\;
\hat\pi_{p|q}\hat\pi_{q|p}\Big]\in\Delta^3.
\end{equation*}

\paragraph{Reliability and target.}
The reliability $\rho_{pq}=n_{pq}/(n_{pq}+n_0)$ with $n_0>0$ shrinks the target toward zero when joint observations are few; with $n_0=5$, the NVS median gives $\rho_{pq}=0.815$ and the PCBA median $0.99998$. Under a signal and noise model given in the appendix, it is the linear shrinkage with minimum mean squared error when $n_0$ equals the noise-to-signal ratio \citep{robbins1956,james1961}; we treat $n_0$ as a hyperparameter. For a triple $(i,p,q)$ with $p\neq q$ and both labels observed,
\begin{equation}
\label{eq:target}
r_{i,p,q}=\rho_{pq}\Big[\,e(y_{i,p},y_{i,q})-b_{pq}\,\Big]\in\mathbb{R}^4,
\qquad
\textstyle\sum_k r_{i,p,q}[k]=0.
\end{equation}
With exact marginals, $\mathbb{E}[r_{i,p,q}]=\mathbf{0}$ when $p$ and $q$ are independent (Proposition~\ref{prop:unbiased}). The signed contrast of the target and the signed score $1-2\hat y_{\text{rel}}$ of a binary head that fits $\mathbb{E}[y_{\text{rel}}]$ have means
\begin{equation}
\label{eq:scores}
\mathbb{E}\big[(r_{00}+r_{11})-(r_{01}+r_{10})\big]=2\rho_{pq}(1-c_{pq})\kappa_{pq},
\qquad
\bar d^{\,\mathrm{bin}}_{pq}=(2c_{pq}-1)+2(1-c_{pq})\kappa_{pq},
\end{equation}
so only the binary score carries the marginal term $2c_{pq}-1$.

\paragraph{Unseen target properties.}
Support sets are stratified by label, so the positive fraction of the target in a support set reflects the sampling rule rather than its marginal. For pairs that involve an unseen target $\tau$, we condition on the support label and subtract only the auxiliary marginal:
\begin{equation*}
b_{i,\tau q}=\Big[(1-y_{i,\tau})(1-\hat\pi_q),\;
(1-y_{i,\tau})\hat\pi_q,\;
y_{i,\tau}(1-\hat\pi_q),\;
y_{i,\tau}\hat\pi_q\Big],
\qquad
\rho_{\tau q}=\frac{n_q}{n_q+n_0},
\end{equation*}
where $n_q$ and $\hat\pi_q$ are the number of observed labels and the smoothed positive rate of $q$ on the meta-training columns. Query labels of the target are never used for relation targets, graph construction, or routing.

\paragraph{Swap-equivariant head.}
Swapping the two properties should exchange the two disagreement coordinates and leave the other two unchanged. We therefore combine the swap-invariant features
\begin{equation*}
\xi_{ipq}=\big[\,h_i\odot(h_p+h_q)\;\Vert\;h_p\odot h_q\;\Vert\;|h_p-h_q|\,\big]\in\mathbb{R}^{3d}
\end{equation*}
with the antisymmetric scalar $\delta_{ipq}=g(h_i\odot h_p)-g(h_i\odot h_q)=-\delta_{iqp}$ from an MLP $g:\mathbb{R}^{d}\to\mathbb{R}$. An MLP $f:\mathbb{R}^{3d}\to\mathbb{R}^3$ maps $\xi_{ipq}$ to $(\hat r_{00},\hat r_{11},u_c)$, where $u_c$ is the sum of the two disagreement coordinates, and $\delta_{ipq}$ splits it:
\begin{equation}
\label{eq:split}
\hat r_{01}=\tfrac{1}{2}(u_c+\delta_{ipq}),
\qquad
\hat r_{10}=\tfrac{1}{2}(u_c-\delta_{ipq}).
\end{equation}
After centering, $\hat r\leftarrow\hat r-\frac{1}{4}\sum_k\hat r[k]$, the head satisfies $\hat r_{iqp}=P_{\text{swap}}\hat r_{ipq}$ exactly, where $P_{\text{swap}}$ exchanges $01$ and $10$ (Proposition~\ref{prop:equivariance}). Unlike the input of ReCoG, $\xi_{ipq}$ contains cross terms between molecule and property, and $\delta_{ipq}$ keeps the direction of the relation. The head is trained with the masked loss $\mathcal{L}_{\text{C4RC}}=\sum m_{i,p,q}\|\hat r_{i,p,q}-r_{i,p,q}\|_2^2\big/\big(4\sum m_{i,p,q}\big)$ over valid triples. The outer objective is $\mathcal{L}_{\text{outer}}=\mathcal{L}_{\text{base}}+\eta_{\text{rel}}\mathcal{L}_{\text{C4RC}}$, where $\mathcal{L}_{\text{base}}$ collects the remaining terms of the ReCoG objective, including the query prediction loss and the bottleneck objective.

\subsection{MSRHA: Molecule-Conditioned Signed Relational Hypergraph Adapter}
\label{sec:5.2}

\paragraph{Signed agreement.}
For a molecule $i$, the episode target $\tau$, and an auxiliary property $a\in A_\tau$, we evaluate the head on the triple and take the signed contrast
\begin{equation}
\label{eq:signed}
d_{i\tau a}=\big(\hat r_{00}+\hat r_{11}\big)-\big(\hat r_{01}+\hat r_{10}\big)\in\mathbb{R}.
\end{equation}
The sign of $d_{i\tau a}$ selects the message channel and its magnitude sets the weight.

\paragraph{Molecule-conditioned sparse routing.}
For each molecule and target, we keep the $k$ auxiliary properties with the highest gated scores:
\begin{equation*}
\beta_{i\tau a}=\lambda_a\,\big|d_{i\tau a}\big|,
\qquad
\mathcal{A}_{i\tau}=\operatorname{Top}\text{-}k_{\,a\in A_\tau}\;\beta_{i\tau a}.
\end{equation*}
Training passes with the bottleneck use $\lambda_a$; ordinary prediction and testing set $\lambda_a=1$.

\paragraph{Ordered hyperedges.}
For each $a\in\mathcal{A}_{i\tau}$, the ordered triple $e=(m_i,\tau,a)$ is encoded jointly:
\begin{equation*}
z_e=\mathrm{ReLU}\Big(\mathrm{LN}\big(W_m h_i+W_t h_\tau+W_a h_a+W_r\hat r_{i\tau a}+b\big)\Big)\in\mathbb{R}^{d_h}.
\end{equation*}
Unlike a sum of pairwise terms $\psi_{12}+\psi_{13}+\psi_{23}$, this encoder can express three-way interactions.

\paragraph{Signed, role-aware write-back.}
Each edge carries an agreement weight $w_e^{+}=\lambda_a\operatorname{ReLU}(d_{i\tau a})$ and an opposition weight $w_e^{-}=\lambda_a\operatorname{ReLU}(-d_{i\tau a})$, and messages are aggregated with projections specific to each sign and node role $\mathrm{rl}(v)\in\{\text{molecule},\text{target}\}$:
\begin{equation}
\label{eq:aggregate}
\Delta h_v=\frac{1}{Z_v}\sum_{e\,\ni\,v}
\Big(w_e^{+}T^{+}_{\mathrm{rl}(v)}(z_e)-w_e^{-}T^{-}_{\mathrm{rl}(v)}(z_e)\Big),
\qquad
Z_v=\varepsilon+\sum_{e\,\ni\,v}\big(w_e^{+}+w_e^{-}\big),
\end{equation}
where $T^{\pm}_{\mathrm{rl}}$ are role-specific affine maps. Dividing by $Z_v$ makes $\Delta h_v$ a weighted average of edge messages, so its scale does not grow with the number of retained edges. Only molecule and target nodes are updated.

\paragraph{Zero-initialized residual gate.}
The update enters as $h'_v=h_v+\gamma\,\Delta h_v$ with $\gamma_0=0$. Because LayerNorm \citep{ba2016} sits inside the message, a zero gate reproduces the predictions and the prediction-loss gradients of the model without the adapter, while $\gamma$ itself still receives a gradient (Proposition~\ref{prop:identity}). Depending on the ablation setting, $\gamma$ is learned or frozen at zero. The adapter scores $N$ candidate triples per molecule and passes messages along $\min(k,N)$ of them, so its cost grows linearly with $N$. On Tox21, C4RC and MSRHA add 588K parameters (19.50\%); Appendix~\ref{sec:G} gives the ablation settings and the complexity.

\section{Experiments}
\label{sec:6}

\subsection{Setup}
\label{sec:6.1}

\paragraph{Data and tasks.}
We use Tox21, SIDER, MUV, ToxCast, and PCBA from MoleculeNet \citep{wu2018} with the property splits of GS-Meta \citep{zhuang2023}. Evaluating the nine ToxCast assay providers separately gives thirteen datasets and 26 conditions in the 1-shot and 10-shot settings. Each episode samples $K$ support molecules per class, and the remaining observed molecules of the target form the query set. In the default warm setting, query molecules keep their auxiliary labels on meta-training properties.

\paragraph{Model and training.}
The encoder is a five-layer pretrained GIN with 300 hidden dimensions, followed by the two-layer context encoder. Training usually runs for 2{,}000 episodes with evaluation every 100 episodes and query batches of 16 for training and 64 for testing. Appendix~\ref{sec:F.1} gives the splits, sampling, evaluation protocol, and hyperparameters.

\paragraph{Baselines and reporting.}
Table~\ref{tab:1} includes methods trained from scratch: Siamese \citep{koch2015}, ProtoNet \citep{snell2017}, MAML \citep{finn2017}, TPN \citep{liu2019}, EGNN \citep{kim2019}, IterRefLSTM \citep{altaetran2017}, and MHNfs \citep{schimunek2023}. Methods with a pretrained encoder are Pre-GNN \citep{hu2020}, Meta-MGNN \citep{guo2021}, Pre-PAR \citep{wang2021}, Pre-GS-Meta \citep{zhuang2023}, Pre-ADKF-IFT \citep{chen2023}, Pre-PACIA \citep{wu2024}, Pre-KRGTS \citep{wang2025a}, Pin-Tuning \citep{wang2024}, and Pre-UniMatch \citep{li2025}. MLTI \citep{yao2022} and SMILE \citep{liu2025b} are augmentation methods, and ReCoG \citep{wang2026} is the closest method. All baseline values and error terms are taken from \citet{wang2026}. We report ROC-AUC (\%) averaged over the held-out properties of each dataset. Tables~\ref{tab:1} and~\ref{tab:3} give the mean and sample standard deviation over three runs with a fixed configuration per entry, and the ablations in Section~\ref{sec:6.4} are single runs.

\subsection{Main Results}
\label{sec:6.2}

\begin{table}[t]
\caption{ROC-AUC (\%). CalibHyper: mean $\pm$ sample standard deviation over three seeds; baselines from \citet{wang2026}. ToxCast averages its nine subsets; a dash marks an unreported result.}
\label{tab:1}
\centering
\resizebox{\linewidth}{!}{%
\setlength{\tabcolsep}{2.5pt}
\begin{tabular}{@{}l*{10}{c}@{}}
\toprule
& \multicolumn{2}{c}{Tox21} & \multicolumn{2}{c}{SIDER} & \multicolumn{2}{c}{MUV} & \multicolumn{2}{c}{ToxCast} & \multicolumn{2}{c}{PCBA} \\
\cmidrule(lr){2-3}\cmidrule(lr){4-5}\cmidrule(lr){6-7}\cmidrule(lr){8-9}\cmidrule(l){10-11}
Method & 10-shot & 1-shot & 10-shot & 1-shot & 10-shot & 1-shot & 10-shot & 1-shot & 10-shot & 1-shot \\
\midrule
Siamese & 80.40$\pm$0.35 & 65.00$\pm$1.58 & 71.10$\pm$4.32 & 51.43$\pm$3.31 & 59.96$\pm$5.13 & 50.00$\pm$0.17 & -- & -- & -- & -- \\
ProtoNet & 74.98$\pm$0.32 & 65.58$\pm$1.72 & 64.54$\pm$0.89 & 57.50$\pm$2.34 & 65.88$\pm$4.11 & 58.31$\pm$3.18 & 63.70$\pm$1.26 & 56.36$\pm$1.54 & 64.93$\pm$1.94 & 55.79$\pm$1.45 \\
MAML & 80.21$\pm$0.24 & 75.74$\pm$0.48 & 70.43$\pm$0.76 & 67.81$\pm$1.12 & 63.90$\pm$2.28 & 60.51$\pm$3.12 & 66.79$\pm$0.85 & 65.97$\pm$5.04 & 66.22$\pm$1.31 & 62.04$\pm$1.73 \\
TPN & 76.05$\pm$0.24 & 60.16$\pm$1.18 & 67.84$\pm$0.95 & 62.90$\pm$1.38 & 65.22$\pm$5.82 & 50.00$\pm$0.51 & 62.74$\pm$1.45 & 50.01$\pm$0.05 & -- & -- \\
EGNN & 81.21$\pm$0.16 & 79.44$\pm$0.22 & 72.87$\pm$0.73 & 70.79$\pm$0.95 & 65.20$\pm$2.08 & 62.18$\pm$1.76 & 63.65$\pm$1.57 & 61.02$\pm$1.94 & 69.92$\pm$1.85 & 62.14$\pm$1.58 \\
IterRefLSTM & 81.10$\pm$0.17 & 80.97$\pm$0.10 & 69.63$\pm$0.31 & 71.73$\pm$0.14 & 49.56$\pm$5.12 & 48.54$\pm$3.12 & -- & -- & -- & -- \\
MHNfs & 80.23$\pm$0.84 & -- & 65.89$\pm$1.17 & -- & 73.81$\pm$2.53 & -- & 74.91$\pm$0.73 & -- & -- & -- \\
\midrule
Pre-GNN & 82.14$\pm$0.08 & 81.68$\pm$0.09 & 73.96$\pm$0.08 & 73.24$\pm$0.12 & 67.14$\pm$1.58 & 64.51$\pm$1.45 & 73.68$\pm$0.74 & 72.90$\pm$0.84 & 76.79$\pm$0.45 & 75.27$\pm$0.49 \\
Meta-MGNN & 82.94$\pm$0.10 & 82.13$\pm$0.13 & 75.43$\pm$0.21 & 73.36$\pm$0.32 & 68.99$\pm$1.84 & 65.54$\pm$2.13 & 76.27$\pm$0.57 & 72.43$\pm$0.85 & 72.58$\pm$0.34 & 72.51$\pm$0.35 \\
Pre-PAR & 84.93$\pm$0.11 & 83.01$\pm$0.09 & 78.08$\pm$0.16 & 74.46$\pm$0.29 & 69.96$\pm$1.37 & 66.94$\pm$1.12 & 75.12$\pm$0.84 & 73.63$\pm$1.00 & 73.71$\pm$0.61 & 72.49$\pm$0.61 \\
Pre-GS-Meta & 86.91$\pm$0.41 & 86.46$\pm$0.55 & 85.08$\pm$0.54 & 84.45$\pm$0.26 & 70.18$\pm$1.25 & 67.15$\pm$2.04 & 83.81$\pm$0.16 & 81.57$\pm$0.18 & 79.40$\pm$0.43 & 78.16$\pm$0.47 \\
Pre-ADKF-IFT & 86.06$\pm$0.35 & 80.97$\pm$0.48 & 70.95$\pm$0.60 & 62.16$\pm$1.03 & 95.74$\pm$0.37 & 67.25$\pm$3.87 & 76.22$\pm$0.13 & 71.13$\pm$1.15 & 80.21$\pm$0.27 & 76.62$\pm$0.73 \\
Pre-PACIA & 86.40$\pm$0.27 & 84.35$\pm$0.14 & 83.97$\pm$0.22 & 80.70$\pm$0.28 & 73.43$\pm$1.96 & 69.26$\pm$2.35 & 76.22$\pm$0.73 & 75.09$\pm$0.95 & 75.36$\pm$0.30 & 70.16$\pm$1.33 \\
Pre-KRGTS & 87.62$\pm$0.29 & 87.54$\pm$0.11 & 85.09$\pm$0.31 & 84.61$\pm$0.16 & 74.47$\pm$0.82 & 68.69$\pm$0.60 & 84.02$\pm$0.10 & 82.39$\pm$0.29 & 81.59$\pm$0.30 & 81.18$\pm$0.17 \\
Pin-Tuning & 91.56$\pm$2.57 & 85.71$\pm$1.33 & 93.41$\pm$3.52 & 84.36$\pm$1.73 & 73.33$\pm$2.00 & 66.44$\pm$2.50 & 84.94$\pm$1.09 & 80.08$\pm$1.21 & 81.26$\pm$0.46 & 71.98$\pm$1.55 \\
Pre-UniMatch & 86.35$\pm$0.13 & -- & 80.34$\pm$0.45 & -- & 86.35$\pm$0.76 & -- & 81.63$\pm$0.73 & -- & -- & -- \\
\midrule
MLTI & 78.95$\pm$0.17 & 65.49$\pm$2.45 & 81.57$\pm$1.72 & 65.85$\pm$1.48 & -- & -- & -- & -- & -- & -- \\
SMILE & 76.23$\pm$0.69 & 60.69$\pm$1.24 & 71.80$\pm$0.79 & 56.49$\pm$1.49 & -- & -- & -- & -- & -- & -- \\
\midrule
ReCoG & 93.52$\pm$1.47 & 91.01$\pm$2.03 & 93.98$\pm$1.53 & 90.06$\pm$1.17 & 95.94$\pm$1.37 & 81.96$\pm$0.42 & 88.85$\pm$1.39 & 86.05$\pm$1.74 & 82.57$\pm$1.42 & 82.39$\pm$1.64 \\
CalibHyper & \textbf{96.82$\pm$1.70} & \textbf{95.96$\pm$2.47} & \textbf{96.89$\pm$2.33} & \textbf{94.93$\pm$1.08} & \textbf{97.19$\pm$1.21} & \textbf{92.93$\pm$1.51} & \textbf{90.01$\pm$0.93} & \textbf{89.64$\pm$1.37} & \textbf{90.37$\pm$2.33} & \textbf{90.15$\pm$1.57} \\
\bottomrule
\end{tabular}}
\end{table}

CalibHyper exceeds the strongest reported baseline in all 26 dataset and shot conditions (Table~\ref{tab:1}; Table~\ref{tab:2} gives the nine ToxCast subsets). Two 10-shot ToxCast entries, CLD and Tanguay, come from an ablation setting with binary relation targets and the MSRHA gate fixed at $\gamma=0$.

The margins follow the label statistics of Section~\ref{sec:4}. On PCBA, where joint observations are abundant and $\rho_{pq}\approx1$, the margin over ReCoG is about eight points in both settings. NVS, whose median pair has only 22 joint observations, is the lowest-scoring subset for both methods and has among the smallest margins. Among the subsets, CEETOX has the second-largest margin in both settings. Its paired increase and decrease assays are strongly negatively dependent (Figure~\ref{fig:2}d); the marginal term $2c_{pq}-1$ partly offsets this dependence in a binary score (Equation~\ref{eq:scores}), whereas the calibrated score keeps the negative sign and MSRHA routes such relations through the negative channel. On Tox21, SIDER, MUV, and eight of the nine ToxCast subsets, the margin is larger in the 1-shot setting, where prediction relies more on the auxiliary context.

\subsection{Component Ablation}
\label{sec:6.3}

\begin{table}[t]
\caption{Component ablation (10-shot ROC-AUC, \%, three runs).}
\label{tab:3}
\centering
\small
\begin{tabular}{@{}lccc@{}}
\toprule
Ablation & Tox21 & SIDER & MUV \\
\midrule
CalibHyper (full) & \textbf{96.82$\pm$1.70} & \textbf{96.89$\pm$2.33} & \textbf{97.19$\pm$1.21} \\
w/o C4RC & 90.08$\pm$2.96 & 90.78$\pm$3.33 & 93.82$\pm$2.19 \\
w/o MSRHA & 87.48$\pm$3.52 & 86.46$\pm$1.92 & 67.23$\pm$1.94 \\
w/o C4RC and MSRHA & 93.97$\pm$2.80 & 89.20$\pm$3.90 & 87.15$\pm$3.53 \\
w/o chance correction & 93.00$\pm$3.13 & 88.39$\pm$1.67 & 87.47$\pm$3.97 \\
w/o sign separation & 95.81$\pm$1.12 & 88.75$\pm$1.93 & 67.43$\pm$2.52 \\
\bottomrule
\end{tabular}
\end{table}

All ablations in Table~\ref{tab:3} use the same three-seed protocol, and the full model has the highest mean on every dataset. Removing MSRHA costs more than removing C4RC on all three datasets, most on MUV, yet neither module helps reliably on its own. With the four-state loss but no adapter, the score falls below that of removing both modules on all three datasets, so a relation loss that only shapes shared representations does not improve prediction. With the adapter but binary targets, the score exceeds that of removing both modules on SIDER and MUV but not on Tox21, where removing both scores above either single removal. The two modules therefore interact rather than add.

Inside the modules, keeping four states without chance correction improves on binary targets only on Tox21, so the four-state encoding alone does not explain the gains. Merging the signed channels gives the best ablated score on Tox21 but drops MUV to the level of the ablation without MSRHA. On MUV, rare positives leave almost no room for negative dependence, so negative scores carry little signal, and a separate channel keeps them apart from positive messages. The unsigned ablation also changes the normalization.

\subsection{Ablations on Auxiliary Information and Routing}
\label{sec:6.4}

\begin{figure}[t]
\centering
\includegraphics[width=0.8\linewidth,height=0.36\textheight,keepaspectratio]{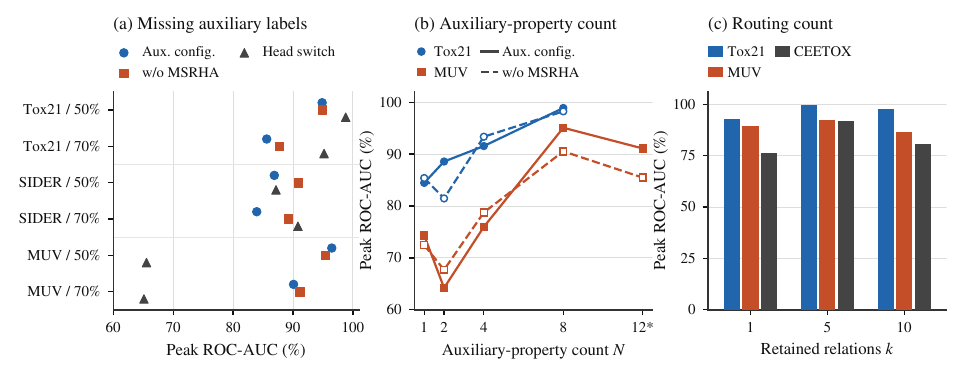}
\caption{Single-run ablations, peak 10-shot ROC-AUC; details in Appendix~\ref{sec:F}. (a) Masked auxiliary labels. (b) Auxiliary-property count $N$ ($12^{*}$: 11 during MUV training). (c) Routing count $k$.}
\label{fig:4}
\end{figure}

\paragraph{Missing auxiliary labels.}
On Tox21, SIDER, and MUV, we hide 50\% or 70\% of the observed auxiliary labels and ablate each dataset's base configuration (``Aux.\ config.'') by removing MSRHA or switching its relation head between binary and four-state targets (Figure~\ref{fig:4}a). Single runs are compared within each dataset, not with the three-run means in Table~\ref{tab:1}. Masking removes joint observations, which lowers $\rho_{pq}$ and makes the calibrated residuals noisier. On Tox21 the binary head gives the highest peak at both rates, while the base configuration and the ablation without MSRHA are nearly tied at 50\%. On MUV, switching to four-state targets drops the peak to Pre-GNN's level at both rates. When positives are rare, the marginal term $2c_{pq}-1$ in Equation~\ref{eq:scores} is close to 1, so binary scores send nearly every auxiliary property through the positive channel with similar weights, whereas the calibrated score is close to zero and routing rests on small residuals. On SIDER, the base configuration freezes $\gamma$ at zero and thus matches the ablation without MSRHA up to the random state, yet their peaks differ by four to five points, the scale of run-to-run variation.

\paragraph{Auxiliary-property count.}
We vary $N$ from 1 to 8 on Tox21 and from 1 to 12 on MUV (Figure~\ref{fig:4}b). More auxiliary properties generally raise the peak, although both MUV curves dip at $N=2$ and decline at $N=12$. With four or fewer auxiliary properties neither setting is consistently better, whereas in the three settings with $N\ge8$ MSRHA gives the higher peak, so the adapter pays off once there are enough relations to route. The advantage holds only at the peak; by the end of training, these runs with MSRHA lose more than those without it.

\paragraph{Routing count.}
The routing count $k$ sets how many relations pass messages to each molecule. Among $k=1$, 5, and 10, Tox21, MUV, and CEETOX all peak at $k=5$ (Figure~\ref{fig:4}c). A single relation discards complementary evidence, which costs most on CEETOX, while $k=10$ admits weakly related properties and, on Tox21, removes selection altogether.

\subsection{Molecular Cases}
\label{sec:6.5}

In the 10-shot setting on Tox21 and on the CEETOX subset of ToxCast, ReCoG \citep{wang2026}, Pin-Tuning \citep{wang2024}, and CalibHyper use the same support sets, query molecules, and auxiliary information, and query labels of the target are hidden before the forward pass. CalibHyper uses unsigned MSRHA on Tox21 and the full model on CEETOX, and the baselines use intermediate checkpoints with different training budgets. From 20{,}895 records of molecule and task pairs, Table~\ref{tab:4} shows four negative molecules, selected post hoc, that CalibHyper classifies correctly and both baselines classify as positive.

\begin{table}[t]
\caption{Molecular cases (10-shot): positive-class score and the label predicted at threshold 0.5 (bold: correct; GT: ground truth). CEETOX target names omit the prefix CEETOX\_H295R\_.}
\label{tab:4}
\centering
\small
\setlength{\tabcolsep}{3pt}
\begin{tabular}{@{}>{\centering\arraybackslash}m{0.6cm}>{\centering\arraybackslash}m{3.6cm}>{\centering\arraybackslash}m{2.2cm}>{\centering\arraybackslash}m{0.8cm}>{\centering\arraybackslash}m{1.3cm}>{\centering\arraybackslash}m{1.5cm}>{\centering\arraybackslash}m{1.6cm}@{}}
\toprule
Case & Molecule & Dataset and target & GT & ReCoG & Pin-Tuning & CalibHyper \\
\midrule
C1 & \includegraphics[width=3.5cm,height=1.5cm,keepaspectratio]{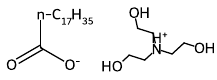} & Tox21\newline SR-HSE & Neg. & 0.99\newline Pos. & 0.59\newline Pos. & $7.39\times10^{-6}$\newline \textbf{Neg.} \\
C2 & \includegraphics[width=3.5cm,height=1.5cm,keepaspectratio]{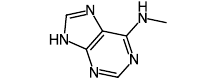} & Tox21\newline SR-HSE & Neg. & 1.00\newline Pos. & 0.57\newline Pos. & $3.77\times10^{-8}$\newline \textbf{Neg.} \\
C3 & \includegraphics[width=3.5cm,height=1.5cm,keepaspectratio]{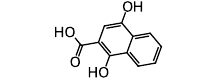} & CEETOX\newline ESTRONE\_up & Neg. & 1.00\newline Pos. & 0.83\newline Pos. & $7.98\times10^{-5}$\newline \textbf{Neg.} \\
C4 & \includegraphics[width=3.5cm,height=1.5cm,keepaspectratio]{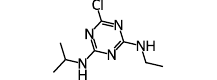} & CEETOX\newline ESTRADIOL\_up & Neg. & 0.98\newline Pos. & 0.87\newline Pos. & 0.02\newline \textbf{Neg.} \\
\bottomrule
\end{tabular}
\end{table}

For C3, the auxiliary labels DOC\_up and PROG\_up are positive while the target ESTRONE\_up is negative. MSRHA routes both relations through the negative channel, so these positive auxiliary labels enter through the projection reserved for opposing relations, the case discussed for CEETOX in Section~\ref{sec:6.2}.

\section{Conclusion}
\label{sec:8}

Label agreement mixes property dependence with class marginals and merges the two ordered disagreement states. CalibHyper replaces it with chance-corrected four-state residuals, shrunk when joint observations are few, and uses the predicted residuals to select, sign, and weight hyperedge messages for each molecule. On 26 conditions, CalibHyper and its ablation settings are competitive with the strongest reported baselines, with larger margins in the 1-shot setting. The margins follow the label statistics: large on PCBA, where joint observations are abundant, and on CEETOX, whose paired assays are negatively dependent, and small on NVS, where joint observations are scarce. The ablations show that the relation loss alone does not improve prediction and that the two modules interact, and routing a few relations per molecule works better than routing one or many.

\paragraph{Limitations.} Labels missing not at random \citep{rubin1976,little2019} and extreme class marginals can bias the relation estimates. On MUV, rare positives and weak dependence leave the calibrated scores near zero, and four-state targets degrade under label masking. Gains vary by dataset and configuration, performance can decline late in training, and the main results cover three seeds.

\subsection*{AI use statement}

We used large language models to assist with writing experiment scripts and orchestrating runs, checking logs, aggregating numerical results, and revising and translating the text of the paper. The authors take responsibility for the methods and experimental conclusions, the accuracy of the content, and the final manuscript.

\subsection*{Ethics statement}

This work uses publicly available molecular property datasets and involves no human subjects or personal information. The intended applications are drug and materials screening.

\subsection*{Reproducibility statement}

Formal definitions and proofs are given in Appendices~\ref{sec:C}, \ref{sec:D}, and \ref{sec:E}, and the experimental setup and full results are given in Appendix~\ref{sec:F}. Appendix~\ref{sec:G} describes the implementation, the ablation settings, and the computing environment.

\bibliography{references}
\bibliographystyle{references}

\clearpage
\appendix
\counterwithin{table}{section}
\counterwithin{equation}{section}
\etocdepthtag.toc{mtappendix}
\etocsettagdepth{mtmain}{none}
\etocsettagdepth{mtappendix}{subsection}
\begingroup
\hypersetup{linkcolor=tocblue}
\etocsettocstyle{\noindent{\large\bfseries List of Appendices}\par\vspace{1.5ex}}{}
\tableofcontents
\endgroup
\clearpage

\section{Extended Related Work}
\label{sec:A}

\subsection{Few-Shot Molecular Property Prediction and Context-Aware Methods}
\label{sec:A.1}

\citet{david2020} review molecular representations. Metric-based methods predict through similarities or prototypes \citep{koch2015,snell2017,altaetran2017}, and gradient-based meta-learning optimizes parameters that adapt well within a task \citep{finn2017,guo2021}. Molecular pretraining methods include Pre-GNN \citep{hu2020} and GROVER \citep{rong2020}, and FS-Mol \citep{stanley2021} provides a separate benchmark for few-shot evaluation. Multitask drug discovery shares statistical strength across the property labels of the same molecule \citep{ramsundar2015,mayr2018}, and context-aware methods carry this idea over to unseen properties by using such auxiliary observations as context.

KRGTS \citep{wang2025a} enriches the relation graph with scaffold and functional-group similarity. CaMol \citep{hoang2026} uses a context graph of functional groups, molecules, and properties to guide atom masking and substructure extraction, and constructs distributional interventions under its assumed causal model. CalibHyper instead estimates how far jointly observed labels deviate from the independence baseline implied by their marginals.

\subsection{Selecting, Weighting, and Compressing Auxiliary Information}
\label{sec:A.2}

Besides MolGroup and AutAuT, WAS \citep{fan2024} decouples the selection and the weighting of pretrained teachers at the instance level and then trains a student by distillation. MolGroup learns a routing mechanism through bilevel optimization, and AutAuT discovers auxiliary tasks with a large language model; the three methods select auxiliary datasets, auxiliary tasks, and pretrained teachers, respectively. CalibHyper works within the auxiliary property set of a given episode and learns relation residuals that generate molecule-conditioned messages.

Graph information bottleneck methods follow the information bottleneck principle \citep{tishby1999} and differ in the conditioning variable and in what they compress. Besides CGIB, IGIB-ISE, and InfoAlign, PGIB \citep{seo2023} combines the compression of key subgraphs with prototype similarity for interpretable prediction. These methods compress substructures of the input graph or molecular representations. Target-conditioned compression of auxiliary property nodes has also been studied in ReCoG \citep{wang2026}. CalibHyper uses this form of context compression together with calibrated relation supervision and signed propagation.

\subsection{Agreement Correction and Relational Structure}
\label{sec:A.3}

\citet{warrens2008} studies which association coefficients for $2\times2$ tables do not depend on the marginal distributions. CalibHyper applies chance correction to the four ordered label states and shrinks the result according to the number of joint observations.

Hypergraph learning can also be approximated with graph structures \citep{yadati2019}, and relational graph convolutions use a separate transformation for each relation type \citep{schlichtkrull2018}. Equivariant networks map each transformation of the input to a fixed transformation of the output \citep{cohen2016}. Residual learning adds an increment to an existing representation \citep{he2016}, and a zero-initialized residual gate lets the new path start from the identity map \citep{bachlechner2021}. Our relation head and residual gate follow these two designs (Propositions~\ref{prop:equivariance} and~\ref{prop:identity}).

\subsection{Evaluation Protocol and Scope of Comparison}
\label{sec:A.4}

Reproducible comparisons require explicit task splits, model configurations, sources of randomness, and aggregation rules \citep{bouthillier2021,pineau2021}. Our main comparison and component analysis use mean ROC-AUC and sample standard deviation across independent runs; Appendix~\ref{sec:F.1} specifies the reporting protocol and distinguishes it from the single-run ablations.

Graph prompting methods and graph foundation models target adaptation across graphs or tasks. OFA \citep{liu2024} unifies feature spaces through text attributes and organizes classification and few-shot tasks at different levels with prompt graphs, GIT \citep{wang2025c} unifies pretraining and in-context learning with task trees, and AutoGFM \citep{chen2025} chooses the architecture according to the features of the downstream graph. We evaluate 1-shot and 10-shot ROC-AUC on unseen properties.

\section{Notation}
\label{sec:B}

Table~\ref{tab:B1} lists the notation of the main text, and Table~\ref{tab:B2} the additional symbols of Appendices~\ref{sec:C} to~\ref{sec:E}.

\begin{table}[ht]
\caption{Notation of the main text.}
\label{tab:B1}
\centering
\small
\begin{tabular}{@{}ll@{}}
\toprule
Symbol & Meaning \\
\midrule
$\mathcal{M}$, $\mathcal{T}$ & set of molecules, set of properties \\
$\mathcal{T}_{\text{train}}$, $\mathcal{T}_{\text{test}}$ & disjoint meta-training and meta-test properties \\
$y_{i,p}\in\{0,1,\star\}$ & label of molecule $i$ for property $p$; $\star$ means unmeasured \\
$o_{i,p}$ & observation indicator $\1[y_{i,p}\neq\star]$ \\
$E_\tau=(S_\tau,Q_\tau)$ & episode with target $\tau$, its support set and query set \\
$A_\tau$, $N$ & auxiliary properties of the episode and their number \\
$K$, $M$ & support molecules per class, query molecules \\
$G_\tau$, $V^m_\tau$, $V^p_\tau$ & context graph, its molecule nodes and property nodes \\
$F_{\text{mol}}$, $d$ & pretrained molecule encoder, embedding width ($d=300$) \\
$F_{\text{rel}}$, $F_{\text{pred}}$ & relation head, property predictor \\
$H=\{h_v\}$, $H'=\{h'_v\}$ & output of the context encoder, representations read by the predictor \\
$\lambda_a$ & bottleneck gate of auxiliary property $a$ \\
$e(u,v)\in\{0,1\}^4$ & one-hot joint label state in the order $(00,01,10,11)$ \\
$n_{pq}$ & number of molecules observed for both $p$ and $q$ \\
$\hat\pi_{p|q}$, $\alpha$ & smoothed positive rate of $p$ on jointly observed molecules, Laplace constant \\
$b_{pq}$ & four-state independence baseline \\
$\rho_{pq}$, $n_0$ & reliability of joint observations, shrinkage constant \\
$r_{i,p,q}\in\mathbb{R}^4$ & chance-corrected four-state relation target, Equation~\ref{eq:target} \\
$\hat r_{i,p,q}$ & prediction of the swap-equivariant head \\
$\xi_{ipq}$, $\delta_{ipq}$ & swap-invariant features, antisymmetric scalar splitting the disagreement states \\
$d_{i\tau a}$ & signed agreement, Equation~\ref{eq:signed} \\
$\beta_{i\tau a}$, $k$ & routing score, auxiliary properties kept per molecule and target \\
$w^{\pm}_e$, $z_e$ & signed hyperedge weights, hyperedge encoding \\
$\mathrm{rl}(v)$ & role of a node (molecule or target), which selects the write-back projection \\
$\gamma$ & residual gate of the adapter, initialized at zero \\
$\eta_{\text{rel}}$ & weight of the relation loss \\
$\kappa_{pq}$, $a_{pq}$, $c_{pq}$ & Cohen's $\kappa$, observed agreement, chance agreement \\
\bottomrule
\end{tabular}
\end{table}

\begin{table}[ht]
\caption{Additional symbols of Appendices~\ref{sec:C} to~\ref{sec:E}.}
\label{tab:B2}
\centering
\small
\begin{tabular}{@{}ll@{}}
\toprule
Symbol & Meaning \\
\midrule
$I_{pq}$, $P_{pq}$ & joint observation set and the empirical law of the pair on it (Definition~\ref{def:pair}) \\
$\pi_{p|q}$ & unsmoothed positive rate of $p$ on jointly observed molecules (Definition~\ref{def:pair}) \\
$b^\ast_{pq}$, $r^\ast_{i,p,q}$ & baseline and ideal target built from exact unsmoothed marginals (Definitions~\ref{def:smooth} and~\ref{def:target}) \\
$\chi_\varnothing$, $\chi_p$, $\chi_q$, $\chi_{pq}$ & Walsh basis of $\mathbb{R}^4$, with $\chi_{pq}=\chi_p\odot\chi_q$ (Definition~\ref{def:walsh}) \\
$s_p$, $s_q$ & labels in $\pm1$ form, $s_p=1-2Y_p$ (Definition~\ref{def:walsh}) \\
$P_{\mathrm{swap}}$, $\mathcal{S}$ & swap acting on output coordinates and on inputs (Definition~\ref{def:swap}) \\
$d(r)$, $\mathrm{dir}(r)$ & signed agreement and directional component (Definition~\ref{def:signed}) \\
$\Delta_{pq}$ & agreement beyond chance, $a_{pq}-c_{pq}$ (Definition~\ref{def:kappa}) \\
$\omega^2$, $\sigma^2$ & prior variance of $\Delta_{pq}$, scale of its sampling variance (Appendix~\ref{sec:E.2}) \\
$\mathcal{P}$, $\Delta^{(3)}_h$ & pairwise decomposable functions, mixed third-order difference (Definition~\ref{def:pairwise}) \\
$\Pi_0$ & zero-sum subspace of $\mathbb{R}^4$ (Appendix~\ref{sec:D.5}) \\
\bottomrule
\end{tabular}
\end{table}

\section{Experimental Setup and Additional Results}
\label{sec:F}

\subsection{Evaluation Protocol and Baseline Sources}
\label{sec:F.1}

\paragraph{Task splits and sampling.} Testing covers all held-out properties: 3 on Tox21, 6 on SIDER, 5 on MUV, and 10 on PCBA, and on the ToxCast subsets 10 on APR, 40 on ATG, 31 on BSK, 4 on CEETOX, 5 on CLD, 39 on NVS, 4 on OT, 20 on TOX21, and 5 on Tanguay. The support set usually contains $K$ molecules per class. When the minority class has at most $K$ molecules, one minority molecule is kept for the query set, and molecules of the other class fill the support set up to $2K$ molecules.

\paragraph{Metrics.} On each dataset, ROC-AUC is averaged over the held-out properties. For ToxCast as a whole, each run averages the nine subsets equally at the same evaluation episode, and Table~\ref{tab:1} reports its three-run mean; the entries of Table~\ref{tab:2} are separate three-run means. AP is computed with the \texttt{average\_precision\_score} function of scikit-learn. Reported differences are calculated before rounding; displayed experimental results and error terms use two decimal places.

\paragraph{Reporting.} The CalibHyper results in Tables~\ref{tab:1}, \ref{tab:3}, and~\ref{tab:2} are means over three independent runs. The model configuration is fixed within each entry; Table~\ref{tab:F0} lists the best-performing hyperparameters of the main experiments. The complete model and ablations use the same evaluation protocol. For run scores $x_1,x_2,x_3$, the reported center is $\bar{x}=\frac13\sum_{s=1}^{3}x_s$ and the $\pm$ term is the sample standard deviation $\sqrt{\frac12\sum_{s=1}^{3}(x_s-\bar{x})^2}$. Baseline numbers retain their original definitions: Pin-Tuning, GS-Meta, and PACIA report means and standard deviations over ten seeds, whereas ReCoG does not specify its repetition count or the definition of its error term. The ablations of Section~\ref{sec:6.4} are single runs with 2{,}000 training episodes and evaluation every 100 episodes, using deterministic evaluation and single-process data loading; they are compared within each group, not with the three-run means.

\paragraph{Archived configuration and checkpoint selection.} The archived searches refined dataset- and shot-specific configurations over successive rounds, varying learning rates, loss weights, smoothing, shrinkage, and adapter settings across CalibHyper and its ablations. Each recorded 2{,}000-episode run was evaluated at 100-episode intervals. The checkpoint with the highest mean ROC-AUC across meta-test properties was selected from these 20 evaluations, and configurations were ranked by this score. No separate validation split was used for selection. The archived three-seed summaries average the scores at the three independently selected checkpoints and report their sample standard deviation. For the archived ToxCast aggregate, the nine subsets are averaged at each common evaluation episode, and the episode with the highest aggregate score is selected. These statistics retain the effect of test-based selection.

\paragraph{Scope of the selection records.} The archived runs and the training loop document this selection procedure, but have not been matched to the individual runs behind the current main-table summaries. Their configuration and checkpoint choices therefore remain unverified by these records.

\paragraph{Baseline sources.} The nineteen compared methods in Table~\ref{tab:1} come from \citet{wang2026}: sixteen baselines and ReCoG from its Table 1, and MLTI and SMILE from its Table 11, which reports only 1-shot and 10-shot ROC-AUC on Tox21 and SIDER. Values, $\pm$ terms, and missing entries follow the original. The eight ToxCast baselines in Table~\ref{tab:2} come from its Tables 8 and 9.

\begin{table}[ht]
\caption{Hyperparameters of the main experiments.}
\label{tab:F0}
\centering
\small
\setlength{\tabcolsep}{4pt}
\begin{tabular}{@{}lr@{\hspace{1.8em}}lr@{\hspace{1.8em}}lr@{}}
\toprule
GIN layers / width $d$ & 5 / 300 & Training episodes & 2{,}000 & Laplace constant $\alpha$ & 1 \\
Query batch (train/test) & 16 / 64 & Routing count $k$ & 5 & Shrinkage constant $n_0$ & 5 \\
\bottomrule
\end{tabular}
\end{table}

\subsection{Missing Auxiliary Labels and Auxiliary-Property Count}
\label{sec:F.2}

Tables~\ref{tab:F1} and~\ref{tab:F2} list the single-run ablations on missing auxiliary labels and on the number of auxiliary properties, 18 runs each (Figure~\ref{fig:4}a, b). Each dataset has its own ablation base (``Aux.\ config.''): C4RC with unsigned MSRHA on Tox21, binary relations with MSRHA on MUV, and binary relations with $\gamma$ fixed at zero on SIDER. The ablations keep the parameters and seeds of each group, and switching the head changes between binary and four-state targets. Peak is the highest value over all evaluations, Last-5 is the mean ROC-AUC of the last five, and Final is the last value.

\begin{table}[ht]
\caption{Missing auxiliary labels in the 10-shot setting (\%).}
\label{tab:F1}
\centering
\small
\setlength{\tabcolsep}{4pt}
\begin{tabular}{@{}llccccc@{}}
\toprule
& & \multicolumn{3}{c}{ROC-AUC} & \multicolumn{2}{c}{AP} \\
\cmidrule(lr){3-5}\cmidrule(l){6-7}
Masked & Ablation setting & Peak & Last-5 & Final & Peak & Final \\
\midrule
\multicolumn{7}{@{}l}{\emph{Tox21}} \\
50\% & C4RC, unsigned MSRHA & 94.87 & 73.73 & 86.32 & 85.53 & 40.31 \\
50\% & w/o MSRHA & 94.89 & 79.70 & 90.07 & 66.13 & 49.88 \\
50\% & Head switched to binary & 98.80 & 88.44 & 80.66 & 93.30 & 30.17 \\
70\% & C4RC, unsigned MSRHA & 85.62 & 72.15 & 57.30 & 41.84 & 22.39 \\
70\% & w/o MSRHA & 87.70 & 75.99 & 87.70 & 44.59 & 44.59 \\
70\% & Head switched to binary & 95.19 & 79.22 & 95.19 & 77.19 & 66.97 \\
\midrule
\multicolumn{7}{@{}l}{\emph{SIDER}} \\
50\% & Binary, $\gamma$ fixed at 0 & 86.88 & 78.16 & 73.01 & 83.68 & 70.07 \\
50\% & w/o MSRHA & 90.88 & 78.47 & 72.40 & 88.23 & 67.58 \\
50\% & Head switched to C4RC & 87.16 & 82.62 & 87.16 & 79.75 & 77.74 \\
70\% & Binary, $\gamma$ fixed at 0 & 83.95 & 68.58 & 40.80 & 80.96 & 57.07 \\
70\% & w/o MSRHA & 89.24 & 82.57 & 89.24 & 85.12 & 85.12 \\
70\% & Head switched to C4RC & 90.82 & 85.11 & 87.56 & 83.53 & 77.80 \\
\midrule
\multicolumn{7}{@{}l}{\emph{MUV}} \\
50\% & Binary, MSRHA & 96.48 & 83.20 & 65.00 & 43.05 & 23.75 \\
50\% & w/o MSRHA & 95.42 & 72.92 & 59.26 & 50.93 & 16.62 \\
50\% & Head switched to C4RC & 65.54 & 60.06 & 58.09 & 1.23 & 0.30 \\
70\% & Binary, MSRHA & 90.10 & 79.79 & 61.91 & 34.71 & 21.75 \\
70\% & w/o MSRHA & 91.19 & 87.97 & 87.02 & 23.22 & 23.22 \\
70\% & Head switched to C4RC & 65.12 & 62.11 & 60.10 & 1.51 & 0.34 \\
\bottomrule
\end{tabular}
\end{table}

\begin{table}[ht]
\caption{Number of auxiliary properties in the 10-shot setting (\%).}
\label{tab:F2}
\centering
\small
\setlength{\tabcolsep}{4pt}
\begin{tabular}{@{}llccccc@{}}
\toprule
& & \multicolumn{3}{c}{ROC-AUC} & \multicolumn{2}{c}{AP} \\
\cmidrule(lr){3-5}\cmidrule(l){6-7}
$N$ & Ablation setting & Peak & Last-5 & Final & Peak & Final \\
\midrule
\multicolumn{7}{@{}l}{\emph{Tox21}} \\
1 & C4RC, unsigned MSRHA & 84.46 & 82.58 & 82.68 & 38.22 & 34.26 \\
1 & w/o MSRHA & 85.39 & 82.44 & 81.59 & 38.17 & 32.28 \\
2 & C4RC, unsigned MSRHA & 88.57 & 83.13 & 82.84 & 47.35 & 36.14 \\
2 & w/o MSRHA & 81.46 & 74.68 & 75.29 & 32.45 & 24.34 \\
4 & C4RC, unsigned MSRHA & 91.55 & 79.00 & 87.60 & 54.61 & 54.52 \\
4 & w/o MSRHA & 93.35 & 86.51 & 90.16 & 55.49 & 45.57 \\
8 & C4RC, unsigned MSRHA & 98.88 & 67.67 & 37.79 & 90.43 & 7.05 \\
8 & w/o MSRHA & 98.23 & 79.63 & 91.55 & 86.83 & 75.20 \\
\midrule
\multicolumn{7}{@{}l}{\emph{MUV}} \\
1 & Binary, MSRHA & 74.33 & 64.33 & 61.49 & 0.88 & 0.37 \\
1 & w/o MSRHA & 72.45 & 62.95 & 61.13 & 0.89 & 0.30 \\
2 & Binary, MSRHA & 64.23 & 53.35 & 52.66 & 0.40 & 0.17 \\
2 & w/o MSRHA & 67.72 & 56.60 & 59.50 & 0.47 & 0.23 \\
4 & Binary, MSRHA & 75.98 & 71.26 & 73.75 & 8.03 & 0.90 \\
4 & w/o MSRHA & 78.76 & 71.44 & 69.17 & 8.84 & 2.37 \\
8 & Binary, MSRHA & 95.08 & 66.27 & 64.28 & 42.10 & 21.43 \\
8 & w/o MSRHA & 90.53 & 85.89 & 82.42 & 28.12 & 23.30 \\
12 & Binary, MSRHA & 91.08 & 54.13 & 44.21 & 38.51 & 8.69 \\
12 & w/o MSRHA & 85.46 & 54.29 & 54.29 & 22.42 & 8.69 \\
\bottomrule
\end{tabular}
\end{table}

The count sweep uses seed 202. On Tox21, $N=8$ uses eight auxiliary properties in both meta-training and testing; on MUV, $N=12$ uses 11 in meta-training and 12 at test time.

Peak and final values often disagree. Many runs reach their peak and then degrade, in some cases below chance level. The drop is not specific to MSRHA: it also occurs without MSRHA, for example on MUV at 50\% masking, and for the SIDER ablation base, whose gate is frozen at zero. In the count sweep with $N\ge8$, however, the runs with MSRHA have the higher peak in all three settings but lose considerably more by the end of training; on Tox21 with $N=8$, the final ROC-AUC of the ablation base falls below chance level while the ablation without MSRHA stays above 90. These peak and final metrics show that stable late training remains an open issue.

\subsection{Routing Count}
\label{sec:F.3}

Table~\ref{tab:F3} lists the single-run peak ROC-AUC of the routing ablation, with other settings fixed.

\begin{table}[ht]
\caption{Peak ROC-AUC (\%) for different routing counts $k$ in the 10-shot setting.}
\label{tab:F3}
\centering
\small
\begin{tabular}{@{}lccc@{}}
\toprule
$k$ & Tox21 & MUV & CEETOX \\
\midrule
1 & 93.01 & 89.43 & 76.38 \\
5 & 99.47 & 92.32 & 91.89 \\
10 & 97.87 & 86.42 & 80.69 \\
\bottomrule
\end{tabular}
\end{table}

\subsection{Results on ToxCast Subsets}
\label{sec:F.4}

Table~\ref{tab:2} uses the same experimental settings as Table~\ref{tab:1}. The best baseline is Pre-KRGTS on NVS and ReCoG elsewhere. Daggers mark the 10-shot CLD and Tanguay entries with inactive residual paths.

\begin{table}[ht]
\caption{ROC-AUC (\%) on the nine ToxCast subsets. CalibHyper: mean over three runs. Baselines are the means reported in Tables 8 and 9 of \citet{wang2026}; method groups follow Table~\ref{tab:1}. $^\dagger$Ablation setting with binary relation targets and the MSRHA gate fixed at $\gamma=0$ (Proposition~\ref{prop:identity}).}
\label{tab:2}
\centering
\small
\setlength{\tabcolsep}{3.2pt}
\begin{tabular}{@{}l*{9}{c}@{}}
\toprule
Method & APR & ATG & BSK & CEETOX & CLD & NVS & OT & TOX21 & Tanguay \\
\midrule
\multicolumn{10}{c}{10-shot} \\
\midrule
ProtoNet & 73.58 & 59.26 & 70.15 & 66.12 & 78.12 & 65.85 & 64.90 & 68.26 & 73.61 \\
MAML & 72.66 & 62.09 & 66.42 & 64.08 & 74.57 & 66.56 & 64.07 & 68.04 & 77.12 \\
EGNN & 80.33 & 66.17 & 73.43 & 66.51 & 78.85 & 71.05 & 68.21 & 76.40 & 85.23 \\
\midrule
Pre-PAR & 86.09 & 72.72 & 82.45 & 72.12 & 83.43 & 74.94 & 71.96 & 82.81 & 88.20 \\
Pre-GS-Meta & 90.15 & 82.54 & 88.21 & 74.19 & 86.34 & 76.29 & 74.47 & 90.63 & 91.47 \\
Pre-KRGTS & 90.31 & 80.12 & 87.92 & 76.63 & 86.97 & 77.52 & 75.11 & 89.83 & 91.73 \\
Pin-Tuning & 92.78 & 83.58 & 89.49 & 75.96 & 87.70 & 76.33 & 75.56 & 90.80 & 92.25 \\
\midrule
ReCoG & 92.97 & 84.37 & 89.95 & 85.85 & 92.63 & 77.19 & 87.50 & 92.35 & 96.83 \\
Best baseline & 92.97 & 84.37 & 89.95 & 85.85 & 92.63 & 77.52 & 87.50 & 92.35 & 96.83 \\
CalibHyper & \textbf{94.51} & \textbf{86.56} & \textbf{92.90} & \textbf{90.58} & \textbf{94.30}$^\dagger$ & \textbf{79.32} & \textbf{92.71} & \textbf{94.21} & \textbf{98.96}$^\dagger$ \\
\midrule
\multicolumn{10}{c}{1-shot} \\
\midrule
ProtoNet & 57.08 & 54.92 & 53.92 & 60.25 & 66.25 & 54.87 & 63.11 & 58.27 & 58.32 \\
MAML & 64.59 & 55.45 & 60.36 & 61.02 & 66.22 & 59.84 & 62.15 & 59.52 & 60.92 \\
EGNN & 67.06 & 57.28 & 60.82 & 60.10 & 71.53 & 56.56 & 66.08 & 63.32 & 74.80 \\
\midrule
Pre-PAR & 84.69 & 70.38 & 79.89 & 66.57 & 77.83 & 72.51 & 70.41 & 80.33 & 86.64 \\
Pre-GS-Meta & 89.49 & 81.69 & 87.28 & 68.55 & 78.69 & 74.36 & 73.56 & 89.46 & 91.10 \\
Pre-KRGTS & 89.45 & 79.54 & 86.84 & 72.50 & 81.21 & 76.63 & 73.68 & 89.51 & 92.15 \\
Pin-Tuning & 87.08 & 80.67 & 83.12 & 71.91 & 79.15 & 71.29 & 73.17 & 84.27 & 90.08 \\
\midrule
ReCoG & 94.39 & 84.92 & 89.23 & 76.71 & 87.79 & 76.29 & 77.89 & 91.79 & 95.40 \\
Best baseline & 94.39 & 84.92 & 89.23 & 76.71 & 87.79 & 76.63 & 77.89 & 91.79 & 95.40 \\
CalibHyper & \textbf{95.76} & \textbf{95.70} & \textbf{93.63} & \textbf{87.61} & \textbf{97.80} & \textbf{79.20} & \textbf{89.32} & \textbf{99.03} & \textbf{99.88} \\
\bottomrule
\end{tabular}
\end{table}

\FloatBarrier
\section{Implementation Details and Computing Environment}
\label{sec:G}

\paragraph{Model.} The relation head and the adapter are built after the base encoder and the predictor are initialized, and both use LayerNorm. With binary targets, the extra relation pass of MSRHA freezes the updates of BatchNorm running statistics but still uses batch statistics and dropout. The C4RC loss is skipped when an episode has no valid triple. Top-$k$ routing is discontinuous where the selected set changes; for a fixed set, the weights of the retained edges are differentiable.

\paragraph{Complexity and overhead.} Let $d_r$ and $d_h$ be the hidden widths of the relation head and the adapter. Counting dense linear layers, scoring candidate triples costs $O(|V^m_\tau|\,N\,d\,d_r)$, and encoding and writing back the retained edges costs $O(|V^m_\tau|\min(k,N)\,d\,d_h)$, plus the cost of top-$k$ selection; by default $d_r=d$. Relation supervision enumerates the valid property pairs on support molecules and the valid auxiliary pairs, excluding the target column, in each query minibatch, so there are at most $O(|V^m_\tau|N^2)$ triples. On Tox21, the backbone and prediction components have 3.02M parameters, while C4RC and MSRHA add 588K (19.50\%), with 271K in the adapter and 317K in the four-state head.

\paragraph{Ablation settings.} Besides the full model, the ablations replace C4RC with binary relation targets (w/o C4RC), remove MSRHA, remove chance correction and shrinkage while keeping four states (w/o chance correction), use unsigned MSRHA, which merges the two signed channels (w/o sign separation), or freeze the MSRHA gate at $\gamma=0$ instead of learning it. The unsigned implementation sets the weight of the positive channel to $\lambda_a|d_{i\tau a}|$ and uses the normalizing mass $\lambda_a(|d_{i\tau a}|+\mathrm{ReLU}(-d_{i\tau a}))$, so edges with negative scores count twice in the denominator; this ablation therefore changes both the sign handling and the normalization. The implementation without chance correction regresses the one-hot state and keeps the zero-sum output. For a zero-sum prediction $\hat r$, $\|\hat r-e\|^2=\|\hat r-(e-\frac14\mathbf{1})\|^2+\frac14$, so it fits the centered one-hot vector up to a constant loss; it also removes the reliability shrinkage.

\paragraph{Gates and evaluation batches.} Auxiliary gates are used only in bottleneck-augmented training and are set to 1 otherwise. The auxiliary runs use relation dropout of 0.1, so exact swap equivariance holds only at evaluation. Benchmark queries are batched in true-label class order; the cases in Section~\ref{sec:6.5} use a fixed random order. Predictions depend on batch composition through the context graph and target-node aggregation.

\paragraph{Environment.} The experiments used eight NVIDIA A800 80GB GPUs with PyTorch 2.7.1, PyTorch Geometric 2.6.1, and Python 3.10. Additional runs used eight Tesla V100S 32GB GPUs with Python 3.9.23 and ran two jobs per GPU.

\section{Formal Definitions and Assumptions}
\label{sec:C}

\subsection{Label Matrix and Empirical Law of a Property Pair}
\label{sec:C.1}

\begin{definition}[Label matrix and observation mask]
The label matrix is $Y\in\{0,1,\star\}^{|\mathcal{M}|\times|\mathcal{T}|}$, and the observation indicator is $o_{i,p}=\1[y_{i,p}\neq\star]$. The symbol $\star$ marks an unmeasured entry, and all sums run over observed entries only.
\end{definition}

\begin{definition}[Joint observation set and empirical law of a pair]
\label{def:pair}
For an ordered pair $(p,q)$ with $p\neq q$, let
\begin{equation*}
I_{pq}=\{\,i: o_{i,p}=o_{i,q}=1\,\},\qquad n_{pq}=|I_{pq}|,
\end{equation*}
and assume $n_{pq}\ge1$. Draw $i\sim\mathrm{Unif}(I_{pq})$ and treat $Y_p:=y_{i,p}\in\{0,1\}$ and $Y_q:=y_{i,q}\in\{0,1\}$ as random variables; their joint law is the empirical law $P_{pq}$ of the pair. Unless stated otherwise, $\mathbb{E}$, $\mathrm{Var}$, and $\mathrm{Cov}$ are taken under $P_{pq}$. The conditional marginals, the observed agreement, and the chance agreement are
\begin{equation*}
\pi_{p|q}=\mathbb{E}[Y_p],\qquad \pi_{q|p}=\mathbb{E}[Y_q],\qquad
a_{pq}=P_{pq}(Y_p=Y_q),
\end{equation*}
\begin{equation*}
c_{pq}=(1-\pi_{p|q})(1-\pi_{q|p})+\pi_{p|q}\,\pi_{q|p}.
\end{equation*}
Proposition~\ref{prop:chance} abbreviates $\pi_{p|q}$ and $\pi_{q|p}$ as $\pi_p$ and $\pi_q$.
\end{definition}

\begin{definition}[Cohen's $\kappa$ and agreement beyond chance]
\label{def:kappa}
If $c_{pq}<1$,
\begin{equation*}
\kappa_{pq}=\frac{a_{pq}-c_{pq}}{1-c_{pq}},
\qquad
\Delta_{pq}:=a_{pq}-c_{pq}=(1-c_{pq})\,\kappa_{pq}.
\end{equation*}
We call $\Delta_{pq}$ the agreement beyond chance.
\end{definition}

\begin{definition}[Smoothed marginals, reliability, and independence baseline]
\label{def:smooth}
Let $n^{+}_{p|q}=|\{\,i\in I_{pq}: y_{i,p}=1\,\}|=n_{pq}\,\pi_{p|q}$. With a Laplace constant $\alpha>0$ and a shrinkage constant $n_0>0$,
\begin{equation*}
\hat\pi_{p|q}=\frac{n^{+}_{p|q}+\alpha}{n_{pq}+2\alpha},
\qquad
\rho_{pq}=\frac{n_{pq}}{n_{pq}+n_0}\in(0,1),
\end{equation*}
\begin{equation*}
b_{pq}=\Big[(1-\hat\pi_{p|q})(1-\hat\pi_{q|p}),\;
(1-\hat\pi_{p|q})\hat\pi_{q|p},\;
\hat\pi_{p|q}(1-\hat\pi_{q|p}),\;
\hat\pi_{p|q}\hat\pi_{q|p}\Big]\in\Delta^3.
\end{equation*}
The baseline built from the unsmoothed marginals $\pi_{p|q}$ and $\pi_{q|p}$ is denoted $b^\ast_{pq}$; it is the product of the marginals of $P_{pq}$. Appendix~\ref{sec:E.2} bounds the smoothing error.
\end{definition}

\begin{definition}[Four-state encoding and Walsh basis]
\label{def:walsh}
The vector $e(u,v)\in\{0,1\}^4$ is the one-hot encoding of the joint state in the order $(00,01,10,11)$, where the first digit is the label of $p$ and the second that of $q$:
\begin{equation*}
e(u,v)=\big[(1-u)(1-v),\;(1-u)v,\;u(1-v),\;uv\big].
\end{equation*}
On $\mathbb{R}^4$ we take the four vectors
\begin{equation*}
\chi_{\varnothing}=(1,1,1,1),\;\;
\chi_{p}=(-1,-1,1,1),\;\;
\chi_{q}=(-1,1,-1,1),\;\;
\chi_{pq}=\chi_p\odot\chi_q=(1,-1,-1,1).
\end{equation*}
They are pairwise orthogonal with $\|\chi\|_2^2=4$ and form the Walsh-Hadamard basis. Labels in $\pm1$ form are $s_p=1-2Y_p$ and $s_q=1-2Y_q$.
\end{definition}

\begin{definition}[C4RC target]
\label{def:target}
For $i\in I_{pq}$,
\begin{equation*}
r_{i,p,q}=\rho_{pq}\big[\,e(y_{i,p},y_{i,q})-b_{pq}\,\big]\in\mathbb{R}^4.
\end{equation*}
Replacing $b_{pq}$ in Equation~\ref{eq:target} with $b^\ast_{pq}$ gives the ideal target $r^\ast_{i,p,q}$.
\end{definition}

\begin{definition}[Swap actions, invariance, and equivariance]
\label{def:swap}
Let $P_{\mathrm{swap}}\in\mathbb{R}^{4\times4}$ be the permutation matrix that exchanges the second and third coordinates, that is, the states $01$ and $10$, and fixes the first and fourth; it is an involution, $P_{\mathrm{swap}}^2=I$. Let $\mathcal{S}(h_i,h_p,h_q)=(h_i,h_q,h_p)$ be the swap on inputs. A map $F:\mathbb{R}^{3d}\to\mathbb{R}^4$ is \emph{swap-invariant} if $F\circ\mathcal{S}=F$ and \emph{swap-equivariant} if $F\circ\mathcal{S}=P_{\mathrm{swap}}F$.
\end{definition}

\begin{definition}[Signed agreement and directional component]
\label{def:signed}
For any $r\in\mathbb{R}^4$,
\begin{equation*}
d(r)=\langle r,\chi_{pq}\rangle=(r_{00}+r_{11})-(r_{01}+r_{10}),
\qquad
\mathrm{dir}(r)=r_{01}-r_{10}=\big\langle r,\tfrac12(\chi_q-\chi_p)\big\rangle.
\end{equation*}
The signed agreement $d$ drives routing and the sign channels and is invariant under the swap; $\mathrm{dir}$ distinguishes the two disagreement states and changes sign under the swap.
\end{definition}

\begin{definition}[Ordered ternary hyperedge and roles]
A hyperedge is an ordered triple $e=(m_i,\tau,a)$ whose entries play the roles of molecule, target, and auxiliary property. The hyperedge reads all three nodes and writes to $\mathrm{wr}(e)=\{m_i,\tau\}$.
\end{definition}

\begin{definition}[Pairwise decomposable class and mixed third-order difference]
\label{def:pairwise}
A function $f:\mathcal{X}_1\times\mathcal{X}_2\times\mathcal{X}_3\to\mathbb{R}$ is \emph{pairwise decomposable} if there exist $\psi_{12}$, $\psi_{13}$, and $\psi_{23}$ such that
\begin{equation*}
f(x_1,x_2,x_3)=\psi_{12}(x_1,x_2)+\psi_{13}(x_1,x_3)+\psi_{23}(x_2,x_3);
\end{equation*}
we denote this class by $\mathcal{P}$. For a step $h$, the forward difference is $(\Delta_j^hf)(x)=f(\dots,x_j+h,\dots)-f(x)$, and the mixed third-order difference is $\Delta^{(3)}_h=\Delta_1^h\Delta_2^h\Delta_3^h$.
\end{definition}

\subsection{Assumptions}
\label{sec:C.2}

\paragraph{(A1) Meaning of expectations.} Appendices~\ref{sec:D} and~\ref{sec:E} take expectations under the empirical law $P_{pq}$ by default. The sampling and shrinkage analysis in Appendix~\ref{sec:E.2} and the analysis of evaluation scores in Appendix~\ref{sec:E.4} use the randomness defined there.

\paragraph{(A2) Superpopulation and i.i.d.\ sampling.} The sampling bound in Appendix~\ref{sec:E.2} assumes that jointly observed label pairs are independent and identically distributed, and the shrinkage analysis uses the prior and conditional noise model given there. Algebraic identities use the empirical law.

\paragraph{(A3) Exact marginals.} Statements about a zero mean under independence use the exact unsmoothed baseline $b^\ast_{pq}$; Appendix~\ref{sec:E.2} bounds the bias caused by replacing it with $b_{pq}$.

\paragraph{(A4) Nondegeneracy.} The analysis of $\kappa_{pq}$ assumes $0<\pi_{p|q},\pi_{q|p}<1$, $c_{pq}<1$, and $n_{pq}\ge1$. The case $c_{pq}=1$ means that both properties are constant with the same value; Cohen's $\kappa_{pq}$ is then undefined, while the C4RC target is still computed with the joint observation mask.

\paragraph{(A5) Inputs and randomness of the head.} The relation head reads context representations $h_i,h_p,h_q\in\mathbb{R}^d$, which may carry support and auxiliary label information. Expressive limits are stated for fixed node representations, and exact equivariance assumes a deterministic forward pass or a shared random mask.

\section{Proofs}
\label{sec:D}

We use the definitions and assumptions of Appendix~\ref{sec:C}.

\subsection{Algebraic Decomposition of the Joint State}
\label{sec:D.1}

\begin{lemma}[Walsh decomposition]
\label{lem:D.1}
With $s_u=1-2u$ and $s_v=1-2v$,
\begin{equation}
\label{eq:walsh}
e(u,v)=\tfrac14\big[\chi_{\varnothing}-s_u\chi_p-s_v\chi_q+s_us_v\chi_{pq}\big].
\end{equation}
\end{lemma}

\begin{proof}
By Definition~\ref{def:walsh}, the inner products of $e(u,v)$ with $\chi_\varnothing$, $\chi_p$, $\chi_q$, and $\chi_{pq}$ are $1$, $-s_u$, $-s_v$, and $s_us_v$. The basis is orthogonal and each vector has squared norm 4, which gives the expansion. Taking the expectation under a product measure with marginals $(\pi_1,\pi_2)$ gives the corresponding coefficients of the independence baseline, whose second-order coefficient is $(1-2\pi_1)(1-2\pi_2)$.
\end{proof}

\subsection{Proof of Proposition~\ref{prop:chance}}
\label{sec:D.2}

Write $t=\mathbb{E}[Y_pY_q]$, $\pi_1=\pi_{p|q}$, and $\pi_2=\pi_{q|p}$. Summing over the four states of two binary labels,
\begin{equation*}
a_{pq}=1-\pi_1-\pi_2+2t,\qquad
c_{pq}=1-\pi_1-\pi_2+2\pi_1\pi_2,
\end{equation*}
so that $a_{pq}-c_{pq}=2(t-\pi_1\pi_2)=2\,\mathrm{Cov}(Y_p,Y_q)$.
Since $y_{\text{rel}}=\1[Y_p\neq Y_q]$, we have $\mathbb{E}[y_{\text{rel}}]=1-a_{pq}=(1-c_{pq})(1-\kappa_{pq})$. Lemma~\ref{lem:D.1} gives Equation~\ref{eq:contrast}, and $\mathrm{Cov}(1-2Y_p,1-2Y_q)=4\,\mathrm{Cov}(Y_p,Y_q)$ gives the remaining claims. \qed

\subsection{Degeneracy of the Pair Term}
\label{sec:D.3}

\begin{proposition}[Degeneracy of the pair term]
\label{prop:degenerate}
Let the input of the relation head be $u=[h_i\Vert h_p]\odot[h_i\Vert h_q]$, where $h_i,h_p,h_q\in\mathbb{R}^d$. Then
\begin{equation}
\label{eq:pair}
u=\big[\,h_i\odot h_i\;\Vert\;h_p\odot h_q\,\big]\in\mathbb{R}^{2d}.
\end{equation}
Hence (i) the first $d$ coordinates do not depend on the property pair; (ii) the last $d$ coordinates do not depend on the molecule; (iii) $u$ is invariant under the swap $p\leftrightarrow q$; (iv) the pre-activation of the first linear layer of the head decomposes additively as $W^{(1)}(h_i\odot h_i)+W^{(2)}(h_p\odot h_q)+b$, with no explicit bilinear cross term between molecule and property; and (v) every swap-equivariant four-state output $F$ computed from $u$ satisfies $F_{01}=F_{10}$ and $\|F-r\|_2^2\ge\frac12(r_{01}-r_{10})^2$ for any ordered target $r$.
\end{proposition}

Later nonlinear layers can still combine molecule and property features, so the additive structure in (iv) concerns only the first pre-activation; the swap invariance in (iii) carries over to every function of $u$.

\begin{proof}
The Hadamard product acts blockwise on the concatenation, which gives Equation~\ref{eq:pair}. Writing the first-layer weight as $[W^{(1)}\;W^{(2)}]$ gives the additive decomposition of the pre-activation.

Since $h_p\odot h_q=h_q\odot h_p$, the input does not change when the two properties are swapped, so every deterministic function $F$ computed from it satisfies $F\circ\mathcal{S}=F$. If the output is also required to be swap-equivariant, then $F=P_{\mathrm{swap}}F$, hence $F_{01}=F_{10}$, and for any ordered target $r$
\begin{equation*}
\|F-r\|_2^2\ge\tfrac12(r_{01}-r_{10})^2.
\end{equation*}
To see this, write $\bar F=F_{01}=F_{10}$ and use $(\bar F-x)^2+(\bar F-y)^2\ge(x-y)^2/2$; the other error terms are nonnegative. The bound continues to hold after taking the expectation over the input distribution.
\end{proof}

\subsection{Expectation of the C4RC Target}
\label{sec:D.4}

\begin{proposition}[Unbiasedness under independence]
\label{prop:unbiased}
If $p$ and $q$ are independent on the jointly observed molecules and the marginals are estimated exactly, then $\mathbb{E}[r_{i,p,q}]=\mathbf{0}$. More generally, with exact marginals and by Equation~\ref{eq:contrast}, the signed contrast of the target satisfies
\begin{equation}
\label{eq:target-contrast}
\mathbb{E}\big[(r_{00}+r_{11})-(r_{01}+r_{10})\big]
=\rho_{pq}\big(\mathbb{E}[s_p s_q]-\mathbb{E}[s_p]\,\mathbb{E}[s_q]\big)
=2\,\rho_{pq}\,(1-c_{pq})\,\kappa_{pq}.
\end{equation}
\end{proposition}

\begin{proof}
With exact joint-observation marginals, the four-state joint law differs from the product measure by $(t-\pi_1\pi_2)(1,-1,-1,1)$, so
\begin{equation*}
\mathbb{E}[r^\ast]=\rho_{pq}\,\mathrm{Cov}(Y_p,Y_q)\,\chi_{pq}.
\end{equation*}
Under independence the covariance vanishes, which gives unbiasedness. Taking the inner product with $\chi_{pq}$, whose squared norm is 4, gives Equation~\ref{eq:target-contrast}. The entries of the one-hot vector and of the baseline each sum to one, so the residual sums to zero at every point, also for the smoothed baseline.

By Lemma~\ref{lem:D.1}, the inner products of the ideal residual with $\chi_p$ and $\chi_q$ are $2\rho_{pq}(Y_p-\pi_1)$ and $2\rho_{pq}(Y_q-\pi_2)$; they vary across molecules and have zero mean. The binary target $y_{\text{rel}}=(1-s_ps_q)/2$, in contrast, is an affine function of the second-order coefficient alone.
\end{proof}

Appendix~\ref{sec:E.2} gives the finite-sample shift of the smoothed estimate.

\subsection{Swap Equivariance of the Relation Head}
\label{sec:D.5}

\begin{proposition}[Exact swap equivariance of a deterministic head]
\label{prop:equivariance}
For fixed deterministic maps $f$ and $g$, every $i,p,q$, and every parameter value, $\hat r_{iqp}=P_{\text{swap}}\hat r_{ipq}$, where $P_{\text{swap}}$ exchanges the coordinates $01$ and $10$ and fixes $00$ and $11$. In addition, $\sum_k\hat r_{ipq}[k]=0$.
\end{proposition}

\begin{proof}
With a deterministic forward pass or a shared mask, $\xi_{iqp}=\xi_{ipq}$ and $\delta_{iqp}=-\delta_{ipq}$, so Equation~\ref{eq:split} only exchanges the coordinates $01$ and $10$. The centering matrix $I-\mathbf{1}\mathbf{1}^\top/4$ commutes with every permutation of coordinates, so it preserves equivariance and makes the output sum to zero. The target transforms in the same way: $e(y_q,y_p)=P_{\mathrm{swap}}e(y_p,y_q)$, $b_{qp}=P_{\mathrm{swap}}b_{pq}$, and $\rho_{qp}=\rho_{pq}$.
\end{proof}

The output assembly covers the zero-sum subspace $\Pi_0=\{x\in\mathbb{R}^4:\sum_kx_k=0\}$: for any $x\in\Pi_0$, take $(\hat r_{00},\hat r_{11},u_c,\delta)=(x_{00},x_{11},x_{01}+x_{10},x_{01}-x_{10})$. The symmetric features can collide, however. The property pairs $((1,3),(2,4))$ and $((1,4),(2,3))$ have the same sum, elementwise product, and absolute difference, and therefore the same symmetric features for every $h_i$. The two pairs are not related by a swap, yet the symmetric channel cannot tell them apart.

\subsection{Identity at Initialization}
\label{sec:D.6}

\begin{proposition}[Identity at initialization]
\label{prop:identity}
Fix the existing parameters, node representations, and random state, and assume that the messages and their derivatives are finite. When $\gamma=0$, $H'=H$, and the input and output of the predictor, as well as the gradient of the current prediction loss with respect to the existing parameters, coincide with those of the same model with the adapter disabled.
\end{proposition}

\begin{proof}
Fix the existing parameters $\theta$, the adapter parameters $\theta_{\mathrm{ad}}$, the node representations, and the random state. Let $\mathcal{L}$ be the prediction loss of the current forward pass, and assume that the messages and the derivatives used are finite. At $\gamma=0$ we have $H'=H+\gamma\Delta H=H$ and
\begin{equation*}
\left.\frac{\partial H'}{\partial\theta}\right|_{\gamma=0}=\frac{\partial H}{\partial\theta},
\qquad
\left.\frac{\partial\mathcal{L}}{\partial\gamma}\right|_{\gamma=0}=\langle\nabla_{H'}\mathcal{L},\Delta H\rangle,
\qquad
\left.\frac{\partial\mathcal{L}}{\partial\theta_{\mathrm{ad}}}\right|_{\gamma=0}=0.
\end{equation*}
The predictor reads the adapter output only through $H'$, so at a zero gate its input, its output, and the prediction loss equal those with the path disabled. With the derivatives above and the chain rule, the prediction gradients of the existing parameters also coincide.
\end{proof}

The identity requires normalization inside the message: if LayerNorm were applied to the residual sum, a zero gate would give $H'=\mathrm{LN}(H)$.

\section{Further Analysis and Scope}
\label{sec:E}

\subsection{Marginal Component of the Binary Target}
\label{sec:E.1}

\paragraph{Risk of constant regression.} For a fixed property pair, let $\mu=\mathbb{E}[y_{\text{rel}}]=(1-c_{pq})(1-\kappa_{pq})$. A constant prediction $v$ has mean squared error $\mu(1-\mu)+(v-\mu)^2$. The prediction $v=1-c_{pq}$, which uses only the marginals, has excess risk $(1-c_{pq})^2\kappa_{pq}^2$ over the best constant, and with fixed marginals the derivative of the target mean with respect to $\kappa_{pq}$ is $-(1-c_{pq})$.

\paragraph{Marginal constraints.} Let $\pi_1=\pi_{p|q}$ and $\pi_2=\pi_{q|p}$. Nonnegativity of the four state probabilities gives $t=P(Y_p=1,Y_q=1)\in[\max(0,\pi_1+\pi_2-1),\min(\pi_1,\pi_2)]$, and substituting into $\kappa_{pq}=2(t-\pi_1\pi_2)/(\pi_1+\pi_2-2\pi_1\pi_2)$ gives the attainable range. For nondegenerate equal marginals $\pi_1=\pi_2=\pi\le1/2$, $\kappa_{pq}\in[-\pi/(1-\pi),1]$, so rare positives leave little room for negative $\kappa$; unequal marginals can also restrict the upper end.

\paragraph{Mean binary score.} Map a binary prediction $\hat y_{\text{rel}}$ to the signed score $1-2\hat y_{\text{rel}}$. The best constant prediction under MSE, or the population average of the best conditional-mean prediction under the same distribution, then gives
\begin{equation*}
\bar d^{\,\mathrm{bin}}_{pq}=1-2\,\mathbb{E}[y_{\text{rel}}]=(2c_{pq}-1)+2(1-c_{pq})\,\kappa_{pq}.
\end{equation*}
The mean score contains the marginal term $2c_{pq}-1$, which is absent from Equation~\ref{eq:target-contrast}.

\subsection{Finite Observations, Shrinkage, and Smoothing}
\label{sec:E.2}

\paragraph{Sampling error.} Under the i.i.d.\ assumption (A2), the observed agreement is a mean of Bernoulli indicators. Hoeffding's inequality \citep{hoeffding1963} gives, with probability at least $1-\zeta$,
\begin{equation*}
|\hat a_{pq}-a^{\mathrm{pop}}_{pq}|\le\sqrt{\frac{\ln(2/\zeta)}{2n_{pq}}},
\end{equation*}
where $a^{\mathrm{pop}}_{pq}$ is the agreement rate of the jointly observed population.

\paragraph{Linear shrinkage.} Suppose that the signal $\Delta$ across property pairs satisfies $\mathbb{E}\Delta=0$ and $\mathrm{Var}(\Delta)=\omega^2>0$, and that its estimate satisfies $\mathbb{E}[\hat\Delta\mid\Delta]=\Delta$ and $\mathrm{Var}(\hat\Delta\mid\Delta)=\sigma^2/n$ with $n,\sigma^2>0$. For the shrunk estimate $\rho\hat\Delta$, expanding the square and using conditional unbiasedness gives
\begin{equation*}
\mathrm{MSE}(\rho)=(1-\rho)^2\omega^2+\rho^2\sigma^2/n,
\qquad
\rho^\star=\frac{n}{n+\sigma^2/\omega^2}.
\end{equation*}
The quadratic is strictly convex, and setting its derivative to zero gives the optimal shrinkage constant $\sigma^2/\omega^2$. In the experiments, $n_0$ is treated as a hyperparameter.

\paragraph{Smoothing shift.} Let $n=n_{pq}$. Subtracting the two estimates in Definition~\ref{def:smooth} gives $\hat\pi-\pi=\alpha(1-2\pi)/(n+2\alpha)$. Each coordinate of the baseline is a product of two factors in $[0,1]$, so the triangle inequality for differences of products gives
\begin{equation*}
\|b_{pq}-b^\ast_{pq}\|_\infty\le\frac{2\alpha}{n+2\alpha}.
\end{equation*}
Since $\mathbb{E}[r]=\mathbb{E}[r^\ast]+\rho(b^\ast-b)$, under independence $\|\mathbb{E}[r]\|_\infty\le2\alpha\rho/(n+2\alpha)$. For the signed contrast, $d(b^\ast)=(1-2\pi_1)(1-2\pi_2)$ and the coefficient of each marginal shift is at most 2 in absolute value, so
\begin{equation*}
\Big|\mathbb{E}[d(r)]-2\rho_{pq}(1-c_{pq})\kappa_{pq}\Big|\le\frac{4\alpha\rho_{pq}}{n_{pq}+2\alpha}.
\end{equation*}
These bounds describe how far the smoothed baseline moves from the exact empirical marginals; shrinkage scales the signal and this shift by the same factor.

\subsection{Expressiveness of the Hyperedge Encoder}
\label{sec:E.3}

\paragraph{Normalization with fixed encodings.} Fix the retained edges and their encodings, and write $U_v$ for the numerator in Equation~\ref{eq:aggregate} and $D_v>0$ for the sum of weights. If the scores or gates are multiplied by a common factor $\nu>0$, positive homogeneity of ReLU gives $\Delta h_v(\nu)=\nu U_v/(\varepsilon+\nu D_v)$. For $\varepsilon=0$ the common scale cancels; for $\varepsilon>0$ the distance from $U_v/D_v$ is $\varepsilon\|U_v/D_v\|/(\varepsilon+\nu D_v)$.

\paragraph{Joint ternary encoding.} On the same input representations, a model that adds single-layer pairwise functions belongs to the class $\mathcal{P}$ of Definition~\ref{def:pairwise} and satisfies $\Delta_1^h\Delta_2^h\Delta_3^hf=0$, because each term misses one variable. The joint encoder can produce a nonzero mixed difference. Take scalar inputs, $W_m=W_t=W_a=1$, and $W_r=b=0$, and omit LayerNorm, so that $z=\mathrm{ReLU}(h_m+h_t+h_a)$. When the three inputs sum to $-2$ and the step is 1,
\begin{equation*}
\Delta^{(3)}_1z=\mathrm{ReLU}(1)-3\,\mathrm{ReLU}(0)+3\,\mathrm{ReLU}(-1)-\mathrm{ReLU}(-2)=1.
\end{equation*}
With LayerNorm, take the two-dimensional pre-activation $(t,-t)$ with $t=h_m+h_t+h_a$, unit gain, zero bias, and $\epsilon_{\mathrm{LN}}>0$. After ReLU, the first coordinate is $\max(t,0)/\sqrt{t^2+\epsilon_{\mathrm{LN}}}$, and the same difference equals $1/\sqrt{1+\epsilon_{\mathrm{LN}}}\neq0$. The class of joint encoders is therefore not contained in $\mathcal{P}$.

\subsection{Label Use and Mean Signed Scores}
\label{sec:E.4}

\paragraph{Labels used in a fixed forward pass.} Fix the episode index, the query batch, the parameters, the support and auxiliary labels, the observation mask, and the random state. The statistics read the meta-training columns, relation supervision on the target and the adapter use support labels of the target, and query graph construction and relation triples read the auxiliary columns. Replacing the values of the query target labels therefore leaves the prediction of this forward pass unchanged.

\paragraph{Mean signed scores.} For auxiliary pairs that use the exact marginals of the same joint-observation distribution, the mean ideal score is $2\rho_{pq}(1-c_{pq})\kappa_{pq}$ (Appendix~\ref{sec:D.4}), so scores of different property pairs are additionally scaled by $\rho_{pq}(1-c_{pq})$.

Pairs of the target and an auxiliary property use the baseline conditioned on support labels, and their mean has to be computed separately. Let the distribution under consideration have marginals $\mu_\tau=\mathbb{E}[Y_\tau]$ and $\mu_a=\mathbb{E}[Y_a]$, and let the auxiliary marginal of the baseline be a fixed $\pi_a^0$. At every point $d(r_{\tau a})=2\rho_{\tau a}(2Y_\tau-1)(Y_a-\pi_a^0)$, and taking the expectation and expanding the covariance gives
\begin{equation*}
\mathbb{E}[d(r_{\tau a})]
=4\rho_{\tau a}\,\mathrm{Cov}(Y_\tau,Y_a)
+2\rho_{\tau a}(2\mu_\tau-1)(\mu_a-\pi_a^0).
\end{equation*}
The second term vanishes when the auxiliary marginal matches $\pi_a^0$ or when the target labels are balanced.

\section{Label Statistics Behind the Relation Model}
\label{sec:H}

Figure~\ref{fig:2}(a, b) averages one statistic per property pair, with equal weight on every pair that has joint observations; when the agreement expected under independence equals 1, we set $\kappa=0$. Panels (c) to (f), and panel D of Figure~\ref{fig:1}, are the four cases in this section. Each case is one limit of supervising relations with binary agreement.

\subsection{Chance Agreement under Sparse Labels: MUV-713 and MUV-733}
\label{sec:H.1}

MUV is derived from PubChem bioactivity data \citep{rohrer2009}. MUV-713 and MUV-733 are screens for inhibitors of coactivator binding to estrogen receptors $\alpha$ and $\beta$ (PubChem AIDs 713 and 733), and 3{,}121 molecules are observed for both in our label matrix. Table~\ref{tab:H1} gives the four joint states. The observed agreement is 99.8718\%, against 99.8719\% under independence, so $\kappa=-0.0005$. The marginals already account for it.

\begin{table}[H]
\caption{Joint label states of MUV-713 and MUV-733.}
\label{tab:H1}
\centering
\small
\begin{tabular}{@{}lcc@{}}
\toprule
State & Observed count & Expected under independence \\
\midrule
Both negative & 3117 & 3117.0 \\
Only MUV-733 positive & 1 & 1.00 \\
Only MUV-713 positive & 3 & 3.00 \\
Both positive & 0 & 0.001 \\
\bottomrule
\end{tabular}
\end{table}

\subsection{Direction of Disagreement: Increase and Decrease of the Same Readout}
\label{sec:H.2}

In the CEETOX panel of ToxCast, the assays OHPROG\_up and OHPROG\_dn (both with the prefix CEETOX\_H295R\_) record increases and decreases of the 17$\alpha$-hydroxyprogesterone readout. The 500 jointly observed molecules give $\kappa=-0.3628$ (Table~\ref{tab:H2}). A binary disagreement label adds the 151 molecules with only a decrease to the 119 with only an increase. The two states are the two directions of the same readout, and the four-state target keeps them apart.

\begin{table}[H]
\caption{Joint label states of the OHPROG increase (up) and decrease (dn) assays.}
\label{tab:H2}
\centering
\small
\begin{tabular}{@{}lcc@{}}
\toprule
 & dn negative & dn positive \\
\midrule
up negative & 230 & 151 \\
up positive & 119 & 0 \\
\bottomrule
\end{tabular}
\end{table}

\subsection{Molecule-Dependent Association: NR-AR and SR-p53 in Tox21}
\label{sec:H.3}

Of the 7{,}823 molecules in Tox21, 243 (3.1\%) have a phosphorus atom in their SMILES. Table~\ref{tab:H3} splits the jointly observed molecules of NR-AR and SR-p53 by that atom. The association is close to zero on the full set and on the molecules without phosphorus, and an order of magnitude larger on the 213 molecules that contain it. An episode-level weight for this property pair would have to use one of these three values for every molecule.

\begin{table}[H]
\caption{Cohen's $\kappa$ between NR-AR and SR-p53 on subsets of Tox21.}
\label{tab:H3}
\centering
\small
\begin{tabular}{@{}lcc@{}}
\toprule
Molecules & Joint observations & Cohen's $\kappa$ \\
\midrule
All jointly observed & 6657 & $+0.018$ \\
Containing phosphorus & 213 & $+0.240$ \\
Other & 6444 & $+0.013$ \\
\bottomrule
\end{tabular}
\end{table}

\subsection{Support from Joint Observations and the Effect of a Single Label}
\label{sec:H.4}

Table~\ref{tab:H4} relabels one joint positive $(1,1)$ as $(1,0)$ on two pairs with $\hat\kappa\approx 0.22$. The MUV pair has counts $[3058,4,3,1]$, so its estimate rests on a single joint positive and falls to $-0.0013$ after relabeling. The PCBA pair has 1{,}151 joint positives among 204{,}894 molecules, and $\hat\kappa$ moves by $1.9\times 10^{-4}$. The two sensitivities differ by a factor of about 1{,}200 (Figure~\ref{fig:2}f). Cohen's $\kappa$ divides by $1-c_{pq}$, which is 0.0029 for the MUV pair. The C4RC target does not: the mean of its signed contrast is $2\rho_{pq}(1-c_{pq})\kappa_{pq}=2\rho_{pq}(a_{pq}-c_{pq})$ (Equation~\ref{eq:scores}), and one relabeled molecule changes $a_{pq}-c_{pq}$ by at most $2/n_{pq}$. The sensitivity of the target therefore follows the number of joint observations, and $\rho_{pq}$ shrinks the target further when that number is small.

\begin{table}[H]
\caption{Effect of relabeling one joint positive as $(1,0)$.}
\label{tab:H4}
\centering
\small
\setlength{\tabcolsep}{4pt}
\begin{tabular}{@{}lccccc@{}}
\toprule
Property pair & Joint observations & Joint positives & Original $\hat\kappa$ & After relabeling & $|\Delta\hat\kappa|$ \\
\midrule
MUV-652, MUV-712 & 3066 & 1 & $+0.2211$ & $-0.0013$ & $2.2\times10^{-1}$ \\
PCBA-1460, PCBA-485364 & 204894 & 1151 & $+0.2231$ & $+0.2230$ & $1.9\times10^{-4}$ \\
\bottomrule
\end{tabular}
\end{table}

A binary target fails each case. It counts the MUV agreement as dependence, it adds the two CEETOX states together, it gives NR-AR and SR-p53 one weight on every molecule, and it has no term that depends on how many joint observations support a relation.

\end{document}

%% file: math_commands.tex
\usepackage{amsmath,amsfonts,bm}

\def\eqref#1{equation~\ref{#1}}

\def\1{\bm{1}}

\DeclareMathAlphabet{\mathsfit}{\encodingdefault}{\sfdefault}{m}{sl}
\SetMathAlphabet{\mathsfit}{bold}{\encodingdefault}{\sfdefault}{bx}{n}